\documentclass[letterpaper, 10pt journal, twoside]{IEEEtran}

\usepackage{amsmath,amssymb,amsfonts,amsthm,amscd,dsfont} 

\usepackage{algorithm,algorithmicx,listings}        
\usepackage[noend]{algpseudocode}
\usepackage{xcolor}
\usepackage{subcaption}

\usepackage{graphicx,tabularx,adjustbox}
\graphicspath{{./fig/}}
\usepackage[font={small}]{caption}
\usepackage{diagbox}
\usepackage[noadjust]{cite}
\usepackage[breaklinks=true, colorlinks, bookmarks=true]{hyperref}
\usepackage{cleveref}

\DeclareMathOperator*{\KL}{KL}
\DeclareMathOperator*{\blkdiag}{blkdiag}

\newcommand{\brl}[1]{\left[#1\right]}

\newtheorem{proposition}{Proposition}

\theoremstyle{definition}

\newtheorem*{assumption*}{Assumption}
\newtheorem*{problem*}{Problem}
\newtheorem{problem}{Problem}
\theoremstyle{remark}

\newcommand{\bfb}{\mathbf{b}}

\newcommand{\bfe}{\mathbf{e}}

\newcommand{\bfp}{\mathbf{p}}

\newcommand{\bfr}{\mathbf{r}}
\newcommand{\bfs}{\mathbf{s}}

\newcommand{\bfu}{\mathbf{u}}
\newcommand{\bfv}{\mathbf{v}}
\newcommand{\bfw}{\mathbf{w}}
\newcommand{\bfx}{\mathbf{x}}
\newcommand{\bfy}{\mathbf{y}}
\newcommand{\bfz}{\mathbf{z}}

\newcommand{\bfmu}{\boldsymbol{\mu}}

\newcommand{\bbE}{\mathbb{E}}

\newcommand{\bbR}{\mathbb{R}}

\newcommand{\calE}{\mathcal{E}}

\newcommand{\calG}{\mathcal{G}}

\newcommand{\calL}{\mathcal{L}}

\newcommand{\calN}{\mathcal{N}}

\newcommand{\calQ}{\mathcal{Q}}

\newcommand{\calV}{\mathcal{V}}

\begin{document}

\title{Multi-Robot Object SLAM Using Distributed Variational Inference}
\author{Hanwen Cao, Sriram Shreedharan, Nikolay Atanasov
	\thanks{Manuscript received: May 1st, 2024; Revised: July 24th, 2024; Accepted: August 19th, 2024.}%
	\thanks{This paper was recommended for publication by Editor Lucia Pallottino upon evaluation of the Associate Editor and Reviewers’ comments.}%
	\thanks{We gratefully acknowledge support from NSF FRR CAREER 2045945 and ARL DCIST CRA W911NF-17-2-0181.}%
	\thanks{The authors are with the Department of Electrical and Computer Engineering, University of California San Diego, La Jolla, CA 92093, USA, e-mails: {\tt\small\{h1cao,\allowbreak sshreedharan,\allowbreak natanasov\}@ucsd.edu}.}%
}

\markboth{IEEE Robotics and Automation Letters. Preprint Version. Accepted August, 2024}
{Cao \MakeLowercase{\textit{et al.}}: Multi-Robot Object SLAM Using Distributed Variational Inference} 

\maketitle

\begin{abstract}
Multi-robot simultaneous localization and mapping (SLAM) enables a robot team to achieve coordinated tasks by relying on a common map of the environment. Constructing a map by centralized processing of the robot observations is undesirable because it creates a single point of failure and requires pre-existing infrastructure and significant communication throughput. This paper formulates multi-robot object SLAM as a variational inference problem over a communication graph subject to consensus constraints on the object estimates maintained by different robots. To solve the problem, we develop a distributed mirror descent algorithm with regularization enforcing consensus among the communicating robots. Using Gaussian distributions in the algorithm, we also derive a distributed multi-state constraint Kalman filter (MSCKF) for multi-robot object SLAM. Experiments on real and simulated data show that our method improves the trajectory and object estimates, compared to individual-robot SLAM, while achieving better scaling to large robot teams, compared to centralized multi-robot SLAM.
\end{abstract}

\section{INTRODUCTION}
\label{sec:introduction}



\IEEEPARstart{S}{imultaneous} localization and mapping (SLAM) \cite{slam_survey} is a fundamental problem for enabling mobile robot to operate autonomously in unknown unstructured environments. In robotics applications, such as transportation, warehouse automation, and environmental monitoring, a team of collaborating robots can be more efficient than a single robot. However, effective coordination in robot teams requires a common frame of reference and a common understanding of the environment \cite{sos-slam}. Traditionally, these requirements have been approached by relying on a central server or lead robot \cite{sris_slam,deutsch2016framework}, which communicates with other robots to receive sensor measurements and update the locations and map for the team. However, communication with a central server requires pre-existing infrastructure, introduces delays or potential estimation inconsistency, e.g., if the server loses track of synchronous data streams, and creates a single point of failure in the robot team. Hence, developing distributed techniques for multi-robot SLAM is an important and active research direction. A fully decentralized SLAM system enables robots to communicate opportunistically with connected peers in an ad-hoc network, removing the need for multi-hop communication protocols and centralized computation infrastructure. It allows flexible addition or removal of robots in the team and, by extending the scalability of the algorithm, enables coverage of larger areas with improved localization and map accuracy. 

\begin{figure}[t]
    \centering
    \includegraphics[width=\linewidth]{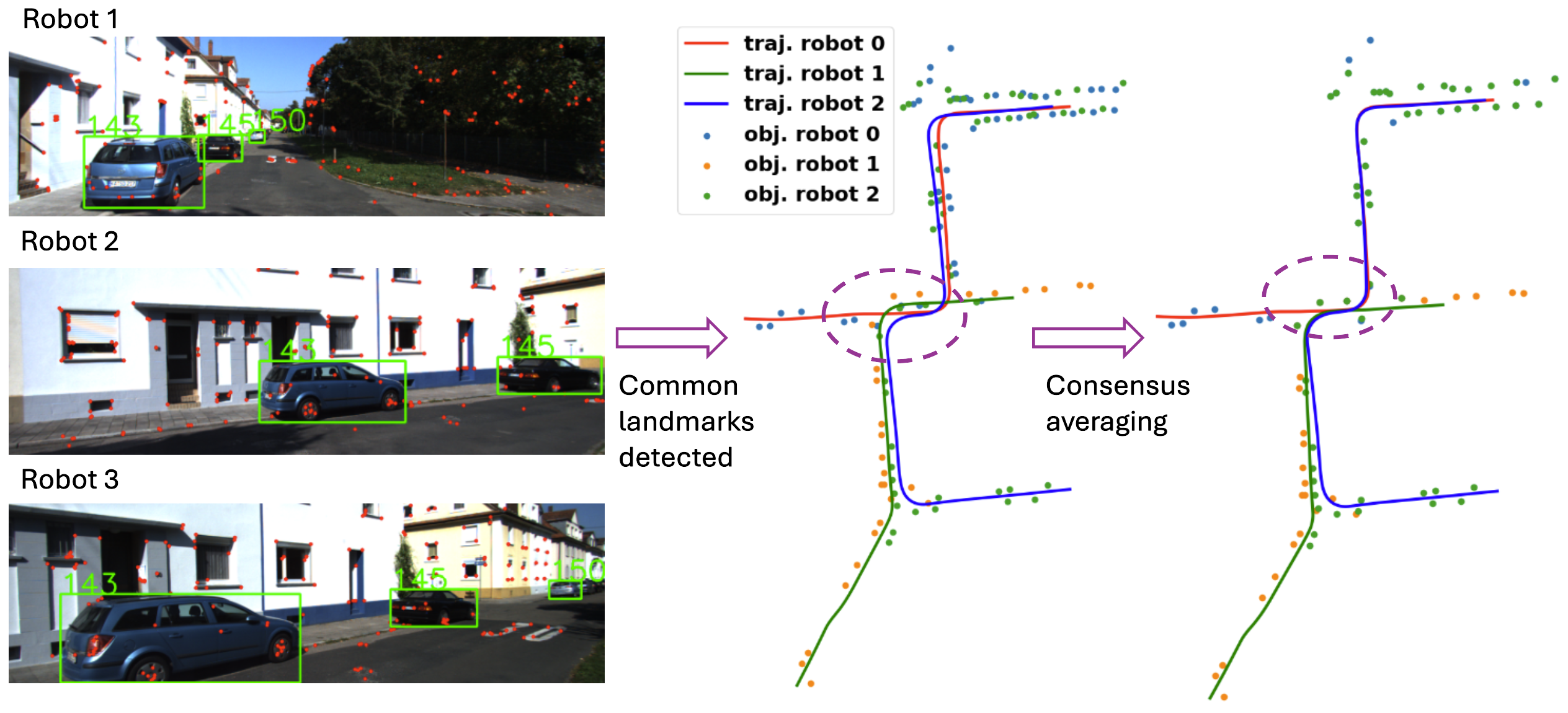}
    \caption{Illustration of multi-robot object SLAM via distributed multi-state constraint Kalman filtering. The left images are inputs for the robots, where the \textcolor{red}{red} are geometric features (extracted by FAST~\cite{fast_feature}) and the \textcolor{green}{green} are object detections (by YOLOv6~\cite{yolov6}). The geometric features and object bounding box centroids are used as observations. When common objects are observed by communicating robots, a consensus averaging step is performed to align the estimated robot trajectories and object positions.}
    \label{fig:teaser}
    \vspace*{-1ex}
\end{figure}

This paper considers a multi-robot landmark-based SLAM problem. We develop an approach for distributed Bayesian inference over a graph by formulating a mirror descent algorithm \cite{mirror_descent} in the space of probability density functions and introducing a regularization term that couples the estimates of neighboring nodes. Our formulation allows joint optimization of common variables (e.g., common landmarks among the robots) and local optimization of others (e.g., private robot trajectories). As a result, each node keeps a distribution only over its variables of interest, enabling both efficient storage and communication. By using Gaussian distributions in the mirror descent algorithm, we derive a distributed version of the widely used multi-state constraint Kalman filter (MSCKF) \cite{msckf} with an additional averaging step to enforce consensus for the common variables. We apply our distributed MSCKF algorithm to collaborative object SLAM using only stereo camera observations at each robot. In the prediction step, each robot estimates its trajectory locally using visual odometry. In the update step, the robots correct their trajectory estimates using both visual features and object detections. As common in the MSCKF, we avoid keeping visual landmarks in the state using a null-space projection step. However, object landmarks are kept as a map representation for each robot and are shared among the robots during the consensus averaging step to collaboratively estimate consistent object maps. In short, each robot estimates its trajectory and an object map locally but communicates with its neighbors to reach agreement on the object maps across the robots. Our contributions are summarized as follows.


\begin{itemize}
    \item We formulate multi-robot landmark SLAM as a variational inference problem over a communication graph with a consensus constraint on the landmark variables.

    \item We develop a distributed mirror descent algorithm with a regularization term that couples the marginal densities of neighboring nodes.

    \item Using mirror descent with Gaussian distributions, we obtain a distributed version of the MSCKF algorithm.

    
    \item We demonstrate multi-robot object SLAM using stereo camera measurements for odometry and object detection on both real and simulated data to show that our method improves the overall accuracy of the trajectories and object maps of robot teams, compared to individual-robot SLAM, while achieving better scaling to large robot teams, compared to centralized multi-robot SLAM\footnote{Code is at \url{https://github.com/ExistentialRobotics/distributed_msckf}.}.
\end{itemize}

\section{RELATED WORK}
\label{sec:related_work}

SLAM is a broad research area including a variety of estimation methods~\cite{msckf,isam,isam2}
as discussed in \cite{slam_survey, huang_vio_survey, vio_benchmark, deepVSLAM, vslam_survey}. Performing SLAM with multiple collaborating robots improves efficiency but also introduces challenges related to distributed storage, computation, and communication. This section reviews recent progress in multi-robot SLAM.


%

\subsection{Multi-robot factor graph optimization} 

Factor graph methods formulate SLAM as an optimization problem over a bipartite graph of variables to be estimated and factors relating variables and measurements via error functions. Tian et al.~\cite{tian2021distributed} propose a certifiably correct pose graph optimization (PGO) method with a novel Riemannian block coordinate descent (RBCD) that operates in a distributed setting. Cunningham et al. \cite{ddf_sam,ddf_sam2} extend smoothing and mapping (SAM) \cite{srsam} by introducing a constrained factor graph that enforces consistent estimates of common landmarks among robots. Choudhary et al. \cite{dpgo} developed a two-stage approach using successive over-relaxation and Jacobi over-relaxation to split the computation among the robots. MR-iSAM2 \cite{mrisam2} extends incremental smoothing and mapping (iSAM2) \cite{isam2} by introducing a novel data structure called mult-root Bayes tree. Tian et al. \cite{tian2023spectral} investigate the relation between Hessians of Riemannian optimization and Laplacians of weighted graphs and design a communication-efficient multi-robot optimization algorithm performing approximate second-order optimization.

Recent SLAM systems utilize the theoretical results of the above works to achieve efficient multi-robot operation. Kimera-Multi~\cite{kimeramulti}, a fully distributed dense metric-semantic SLAM system, uses a two-stage optimization method built upon graduated non-convexity \cite{gnc} and RBCD \cite{tian2021distributed}. DOOR-SLAM \cite{door_slam} uses \cite{dpgo} as a back-end and pairwise consistency maximization \cite{pcm} for identifying consistent measurements across robots. Xu et al. \cite{xu2022d2slam} develop a distributed visual-inertial SLAM combining collaborative visual-inertial odometry with an alternating direction method of multipliers (ADMM) algorithm and asynchronous distributed pose graph optimization \cite{arock}. Andersson et al. \cite{sris_slam} design a multi-robot SLAM system built upon square-root SAM \cite{srsam} by utilizing rendezvous-measurements.

\subsection{Multi-robot filtering}


Filtering methods, such as the Kalman filter, offer a computationally lightweight alternative to factor graph optimization by performing incremental prediction and update steps that avoid a large number of iterations. Roumeliotis and Bekey \cite{multirobot_localization} showed that the Kalman filter equations can be written in decentralized form, allowing decomposition into smaller communicating filters at each robot. Thrun et. al \cite{multirobot_eif_slam} presented a sparse extended information filter for multi-robot SLAM, which actively removes information to ensure sparseness at the cost of approximation. With nonlinear motion and observation models, a decentralized extended Kalman filter (EKF) has an observable subspace of higher dimension than the actual nonlinear system and generates unjustified covariance reduction \cite{oc_ekf}. Huang et al. introduced observability constraints in EKF \cite{oc_ekf} and unscented smoothing \cite{clatt} algorithms to ensure consistent estimation. Gao et al. \cite{phd_slam} use random finite sets to represent landmarks at each robot and maintain a probability hypothesis density (PHD). The authors prove that geometric averaging of the robot PHDs over one-hop neighbors leads to convergence of the PHDs to a global Kullback-Leibler average, ensuring consistent maps across the robots.
Zhu et al.~\cite{covariance_intersection} propose a distributed visual-inertial cooperative localization algorithm by leveraging covariance intersection to compensate for unknown correlations among the robots and deal with loop-closure constraints.

Our contribution is to derive a fully distributed filter for object SLAM from a constrained variational inference perspective. Our formulation makes a novel connection to distributed mirror descent and enables robots to achieve landmark consensus efficiently with one-hop communication only and without sharing private trajectory information.

\section{PROBLEM STATEMENT}
\label{sec:problem_statement}

Consider $n$ robots seeking to collaboratively construct a model of their environment represented by a variable $\bfy$, e.g., a vector of landmark positions. Each robot $i$ also aims to estimate its own state $\bfx_{i,t}$, e.g., pose, at time $t$. The combined state of robot $i$ is denoted as $\bfs_{i,t} = [\bfx_{i,t}^\top \ \bfy^\top]^\top$ and evolves according to a known Markov motion model:
\begin{equation} \label{eq:motion_model}
  \bfs_{i,t+1} \sim f_i( \cdot \mid \bfs_{i,t}, \bfu_{i,t}),
\end{equation}
where $\bfu_{i,t}$ is a control input and $f_i$ is the probability density function (PDF) of the next state $\bfs_{i,t+1}$. Each robot receives observations $\bfz_{i,t}$ according to a known observation model:
\begin{equation} \label{eq:observation_model}
  \bfz_{i,t} \sim h_i( \cdot \mid \bfs_{i,t}),
\end{equation}
where $h_i$ is the observation PDF.

The robots communicate over a network represented as a connected undirected graph $\calG = (\calV, \calE)$ with nodes $\calV = \{1, \ldots, n\}$ corresponding to the robots and edges $\calE \subseteq \calV \times \calV$ specifying robot pairs that can exchange information, e.g. $(i,j) \in \calE$ indicates that robot $i$ and $j$ can exchange information. Let $A \in \bbR^{n \times n}$ be a doubly stochastic weighted adjacency matrix of $\calG$ such that $A_{ij} > 0$ if $(i,j) \in \calE$ and $A_{ij} = 0$ otherwise. Also, let $\calN_i := \{ j \in \calV | (j,i) \in \calE\} \cup \{i\}$ denote the set of \emph{one-hop neighbors} of robot $i$. \
and includes node $i$ itself. 
We consider the following problem.

\begin{problem}\label{problem}
Given control inputs $\bfu_i = [\bfu_{i,0}^\top,\ldots,\bfu_{i,T-1}^\top]^\top$ and observations $\bfz_{i} := [\bfz_{i,0}^\top,\ldots,\bfz_{i,T}^\top]^\top$, each robot $i$ aims to estimate the robot states $\bfs_i := [\bfs_{i,0}^\top,\ldots,\bfs_{i,T}^\top]^\top$ collaboratively by exchanging information only with one-hop neighbors $j \in \calN_i$ in the communication graph $\calG$.
\end{problem}



\section{Distributed Variational Inference}
\label{sec:dvi}

We approach the collaborative estimation problem using variational inference. We develop a distributed mirror descent algorithm to estimate a (variational) density of the states $\bfs_{i,t}$ with regularization that enforces consensus on the estimates of the common landmarks $\bfy$ among the robots.

\subsection{Variational inference}



As shown in \cite{barfoot2020exactly}, a Kalman filter/smoother can be derived by minimizing the Kullback-Leibler (KL) divergence between a variational density $q_i(\bfs_i)$ and the true Bayesian posterior $p_i(\bfs_i | \bfu_i, \bfz_i)$. Adopting a Bayesian perspective, the posterior is proportional to the joint density, which factorizes into products of motion and observation likelihoods due to the Markov assumptions in the models \eqref{eq:motion_model}, \eqref{eq:observation_model}:
\begin{align}
p_i(&\bfs_i | \bfu_{i}, \bfz_{i}) \propto  p_i(\bfs_i, \bfu_{i}, \bfz_{i}) \label{eq:likelyhood_decomposition} \\
&\propto p_i(\bfs_{i,0}) \prod_{t=0}^{T-1} f_i(\bfs_{i,t+1} | \bfs_{i,t}, \bfu_{i,t}) \prod_{t = 0}^{T} h_i(\bfz_{i,t} | \bfs_{i,t}). \notag
\end{align}
%
The KL divergence between the variational density $q_i(\bfs_i)$ and the true posterior $p_i(\bfs_i | \bfu_i, \bfz_i)$ can be decomposed as:
\begin{equation}
\begin{aligned}
\KL(q_i || p_i) &= \bbE_{q_i} [- \log p_i(\bfs_i, \bfu_i, \bfz_i)]\\
& - \underbrace{\bbE_{q_i}[-\log q_i(\bfs_i)]}_{\text{entropy}} + \underbrace{\log p_i(\bfu_i, \bfz_i)}_{\text{constant}},
\end{aligned}
\end{equation}
where for simplicity of notation $q_i$ without input arguments refers to $q_i(\bfs_i)$. Dropping the constant term, leads to the following optimization problem at robot $i$:
\begin{equation}\label{eq:original_objective}
\min_{q_i\in \calQ_i} c_i(q_i) := \bbE_{q_i} [- \log p_i(\bfs_i, \bfz_i, \bfu_i) + \log q_i(\bfs_i)],
\end{equation}
where $\calQ_i$ is a family of admissible variational densities.

\subsection{Distributed mirror descent}

We solve the variational inference problem in \eqref{eq:original_objective} using the mirror descent algorithm \cite{mirror_descent}. We use mirror descent because it includes an explicit (Bregman divergence) regularization term in the objective function that can incorporate information from one-hop neighbors \cite{paritosh2022distributed}. This allows us to formulate a distributed version of mirror descent that enforces agreement among the landmark estimates of different robots with convergence guarantees. 
Mirror descent is a generalization of projected gradient descent that performs projection using a generalized distance (Bregman divergence), instead of the usual Euclidean distance, to respect the geometry of the constraint set $\calQ_i$. Since $\calQ_i$ is a space of PDFs, a suitable choice of Bregman divergence is the KL divergence. Starting with a prior PDF $q_i^{(0)}(\bfs_i)$, the mirror descent algorithm performs the following iterations:
\begin{equation} \label{eq:single_Mirror_Descent}
q_i^{(k+1)} \in \arg\min_{q_i \in \calQ_i} \bbE_{q_i}\brl{ \frac{\delta c_i}{\delta q_i} (q_i^{(k)})} + \frac{1}{\alpha_{k}} \KL(q_i || q_i^{(k)}),
\end{equation}
where $\delta c_i / \delta q_i(q_i^{(k)})$ is the Fr{\'e}chet derivative of $c_i(q_i)$ with respect to $q_i$ evaluated at $q_i^{(k)}$ and $\alpha_{k} > 0$ is the step size.

Note that the optimizations \eqref{eq:single_Mirror_Descent} at each robot $i$ are completely decoupled and, hence, each robot would be estimating its own density over the common landmarks $\bfy$. To make the estimation process collaborative, the regularization term $\KL(q_i || q_i^{(k)})$ in \eqref{eq:single_Mirror_Descent} should require that the PDF $q_i$ of robot $i$ is also similar to the priors $q_j^{(k)}$ of its neighbors $\calN_i$ rather than its own prior $q_i^{(k)}$ alone. 
In our case, the PDFs $q_j^{(k)}(\bfs_j)=q_j^{(k)}(\bfx_j,\bfy)$ are not defined over the same set of variables since each robot $j$ is estimating its own private state $\bfx_j$ as well. Inspired by but different from~\cite{paritosh2022distributed}, to enforce consensus only on the common state $\bfy$, the KL divergence term in \eqref{eq:single_Mirror_Descent} can be decomposed as a sum of marginal and conditional terms:
\begin{align}
\KL(q_i(\bfx_i,\bfy) || q_i^{(k)}(\bfx_i,\bfy)) &= \KL(q_i(\bfy) || q_i^{(k)}(\bfy)) \\
&\;+ \KL(q_i(\bfx_i|\bfy) || q_i^{(k)}(\bfx_i|\bfy)). \notag
\end{align}
Hence, we can regularize only the marginal density $q_i(\bfy)$ of the common environment state $\bfy$ to remain similar to the marginal densities $q_j^{(k)}(\bfy)$ of the one-hop neighbors by using a weighted sum of KL divergences. This leads to the following optimization problem at robot $i$:
\begin{align}
q_i^{(k+1)}& \in \arg\min_{q_i \in \calQ_i} g_i(q_i) \label{eq:distributed_MD}\\
g_i(q_i)& := \bbE_{q_i}\brl{ \frac{\delta c_i}{\delta q_i} (q_i^{(k)})} + \frac{1}{\alpha^{(k)}} \KL(q_i(\bfx_i|\bfy) || q_i^{(k)}(\bfx_i|\bfy)) \notag\\
&\quad + \frac{1}{\alpha^{(k)}} \sum_{j \in \calN_i} A_{ij} \KL(q_i(\bfy) || q_j^{(k)}(\bfy)),\notag
\end{align}
where $\calQ_i = \{q_i \mid \int q_i = 1 \}$ is the feasible set and $A_{ij}$ are the elements of the adjacency matrix with $\sum_{j \in \calN_i} A_{ij} = 1$. We derive a closed-form expression for the optimizer in the following proposition.

\begin{proposition}\label{prop:DM_solution}
    The optimizers of \eqref{eq:distributed_MD} satisfy:
    \begin{equation}\label{eq:DM_solution}
    \begin{aligned}
	q^{(k+1)}_{i} (\bfx_i, \bfy) &\propto [p_i (\bfx_i, \bfy, \bfz_i, \bfu_i) / q_i^{(k)}(\bfx_i, \bfy)]^{\alpha_{k}} \\
	&\quad \ \quad q_i^{(k)}(\bfx_i | \bfy) \prod_{j \in \calN_{i}} [ q^{(k)}_j (\bfy) ]^{A_{ij}}.
    \end{aligned}
    \end{equation}
\end{proposition}

\begin{proof}
 See Appendix~\ref{sec:proof_DM_solution}.
\end{proof}

\subsection{Linear Gaussian case}

In this section, we consider linear Gaussian models and obtain an explicit form of the distributed variational inference update in \eqref{eq:DM_solution}. Suppose each robot $i$ has the following motion and observation models:
\begin{equation}
    \begin{aligned}
    \bfs_{i,t+1} &= F_i \bfs_{i,t} + G_i \bfu_{i,t} + \bfw_{i,t}, \ &\bfw_{i,t}& \sim \calN(\mathbf{0}, W_i), \\
    \bfz_{i,t} &= H_i \bfs_{i,t} + \bfv_{i,t}, \ &\bfv_{i,t}& \sim \calN(\mathbf{0}, V_i).
    \end{aligned}
\end{equation}
Let the prior density of $\bfs_{i,0}$ be $\calN(\boldsymbol{\mu}_{i,0}, \Sigma_{i,0})$ and considering all timesteps, we can write the models in lifted form as
\begin{equation}\label{eq:lifted_models}
\begin{aligned}
    \bfs_i &= \bar{F}_i (\bar{G}_i \bfu_i + \bar{\bfw}_i), \ \Bar{\bfw}_i \sim \calN(\mathbf{0}, \bar{W}_i), \\
    \bfz_i &= \Bar{H}_i \bfs_i + \Bar{\bfv}_i, \ \bfv_i \sim \calN(\mathbf{0}, \Bar{V}_i),
\end{aligned}
\end{equation}
with the lifted terms defined as below
\begin{align}
    & \bfs_i = [\bfs_{i,0}^\top \ \cdots \ \bfs^\top_{i,T}]^\top, \;\; \bfu^\top_i = [\boldsymbol{\mu}_{i,0}^\top \ \bfu_{i,0} \ \cdots \ \bfu^\top_{i,T}]^\top,\notag\\
    & \bfz_i = [\bfz_{i,0}^\top \ \cdots \ \bfz^\top_{i,T}]^\top,\;\; \Bar{H}_i = I_{T+1} \otimes H_i, \notag \\
    & \Bar{V}_i = I_{T+1} \otimes V_i, \Bar{W}_i = 
    \begin{bmatrix}
        \Sigma_{i, 0} & 0 \\
        0                   & I_{T} \otimes W_i
    \end{bmatrix},  \notag \\
    &\Bar{F}_i = 
    \begin{bmatrix}
        I       & 0     &\cdots &0 \\
        F       & I     &\cdots &0 \\
        \vdots  &\vdots &\ddots &\vdots \\
        F^T     &F^{T-1}&\cdots & I        
    \end{bmatrix}, 
    \Bar{G}_i = 
    \begin{bmatrix}
    I   & 0 \\
    0   & I_{T} \otimes G_i
    \end{bmatrix},
\end{align}
%
%
where $\otimes$ is the Kronecker product. Denoting the density and distribution at iteration $k$ as $q_i^{(k)}(\bfs_i)$ and $\calN(\boldsymbol{\mu}_{i,(k)}, \Sigma_{i,(k)})$, the distributed variational inference update in \eqref{eq:DM_solution} is computed in the following propositions.

\begin{proposition}
    \label{prop:construct_joint}
    Consider a joint Gaussian distribution
    \begin{equation}
        \begin{bmatrix}
            \bfx \\ \bfy
        \end{bmatrix} \sim
        \calN \left( 
        \begin{bmatrix}
         \boldsymbol{\mu}^{\bfx} \\ \boldsymbol{\mu}^{\bfy}   
        \end{bmatrix},
        \begin{bmatrix}
         \Sigma^{\bfx} & \Sigma^{\bfx\bfy} \\
         \Sigma^{\bfx\bfy\top} & \Sigma^{\bfy}
        \end{bmatrix} \right).
    \end{equation}
    If the marginal distribution over $\bfy$ changes to $\calN(\bar{\boldsymbol{\mu}}^{\bfy}, \bar{\Sigma}^{\bfy})$, the new joint distribution $\calN(\bar{\boldsymbol{\mu}}, \bar{\Sigma})$ of $(\bfx,\bfy)$ will be
    \begin{align}
        &\bar{\boldsymbol{\mu}} = 
	\begin{bmatrix}
	A \bar{\boldsymbol{\mu}}^{\bfy} + \bfb \\
	\bar{\boldsymbol{\mu}}^{\bfy}
	\end{bmatrix},
	\ \bar{\Sigma} = 
	\begin{bmatrix}
	A \bar{\Sigma}^{\bfy} A^\top + P & A \bar{\Sigma}^{\bfy} \\
	\bar{\Sigma}^{\bfy} A^\top & \bar{\Sigma}^{\bfy}
	\end{bmatrix}, \notag \\
	&A = \Sigma^{\bfx \bfy} \Sigma^{\bfy -1}, \quad\;\; \bfb = \boldsymbol{\mu}^{\bfx} - \Sigma^{\bfx \bfy} \Sigma^{\bfy -1} \boldsymbol{\mu}^{\bfy}, \notag \\ 
	&P = \Sigma^{\bfx} - \Sigma^{\bfx} \Sigma^{\bfy -1} \Sigma^{\bfx \bfy \top}.
     \end{align}
\end{proposition}

\begin{proposition}
    \label{prop:VI_Gaussian_sol}
    In the linear Gaussian case, the distributed variational inference update in \eqref{eq:DM_solution} can be obtained by first averaging marginal densities of the common state $\bfy$ across the neighbors $\calN_i$ of robot $i$:
    %
    \begin{equation}
        \bar{\Sigma}^{\bfy -1}_{i,(k)} = \sum_{j \in \mathcal{N}_i} A_{ij} \Sigma^{\bfy -1}_{j,(k)}, \ 
        \bar{\Sigma}^{\bfy -1}_{i,(k)} \bar{\boldsymbol{\mu}}^{\bfy}_{i,(k)} = \sum_{j \in \mathcal{N}_i} A_{ij} \Sigma^{\bfy -1}_{j,(k)} \boldsymbol{\mu}^{\bfy}_{j,(k)},
         \label{eq:MD_averaging}
     \end{equation}
     then constructing a new joint distribution $\calN(\bar{\boldsymbol{\mu}}_{i,(k)}, \bar{\Sigma}_{i,(k)})$ according to Proposition~\ref{prop:construct_joint}, and finally updating the density using the motion and observation models in \eqref{eq:lifted_models}:
     \begin{align}
        &\Sigma^{-1}_{i,(k+1)} \!= \bar{\Sigma}^{-1}_{i,(k)} \!+ \alpha_{k} (\Bar{F}^{-\top}_i \Bar{W}_i^{-1} \Bar{F}^{-1}_i \!+ \Bar{H}^\top_i \Bar{V}_i^{-1} \Bar{H}_i -\! \Sigma^{-1}_{i,(k)}), \notag \\
        &\Sigma^{-1}_{i,(k+1)} \boldsymbol{\mu}_{i,(k+1)} = \bar{\Sigma}^{-1}_{i,(k)}\bar{\boldsymbol{\mu}}_{i,(k)} \notag \\
        &\ \ + \alpha_{k} (\Bar{F}_i^{-\top} \Bar{W}_i^{-1} \Bar{G}_i \bfu_i + \Bar{H}_i^\top \Bar{V}_i^{-1} \bfz_i - \Sigma^{-1}_{i,(k)}\boldsymbol{\mu}_{i,(k)}). \label{eq:VI_Gaussian_update}
    \end{align}
\end{proposition}

\begin{proof}
     See Appendix~\ref{sec:proof_VI_solution}.
\end{proof}


In the above proposition, the averaging over the marginal densities in \eqref{eq:MD_averaging} comes from the term $\prod_{j \in \calN_{i}} [ q^{(k)}_j (\bfy) ]^{A_{ij}}$ in \eqref{eq:DM_solution}, which enforces consensus among the robots over the common variables $\bfy$. If the robots are in consensus, i.e., $\boldsymbol{\mu}^{\bfy}_{i,(k)}=\boldsymbol{\mu}^{\bfy}_{j,(k)}, \Sigma^{\bfy}_{i,(k)}=\Sigma^{\bfy}_{j,(k)}, \ \forall j \in \calN_i$, \eqref{eq:VI_Gaussian_update} with $\alpha_k=1$ converges in just one step, 
\begin{equation}
    \begin{aligned}
    \Sigma^{-1}_{i} &= \Bar{F}^{-\top}_i \Bar{W}_i^{-1} \Bar{F}^{-1}_i + \Bar{H}^\top_i \Bar{V}_i^{-1} \Bar{H}_i, \\
    \Sigma^{-1}_{i} \boldsymbol{\mu}_{i} &=  
    \Bar{F}_i^{-\top} \Bar{W}_i^{-1} \Bar{G}_i \bfu_i + \Bar{H}_i^\top \Bar{V}_i^{-1} \bfz_i.
    \end{aligned}
    \label{eq:regular_KF_info_space}
\end{equation}
As shown in \cite[Ch.~3.3]{state_barfoot}, considering only two consecutive time steps in the lifted form in \eqref{eq:lifted_models} leads to a Kalman filter. Using the result in Proposition \ref{prop:VI_Gaussian_sol}, we obtain a distributed Kalman filter that incorporates the consensus averaging step in \eqref{eq:MD_averaging}. To allow correlation between the motion and measurement noise, we follow Crassidis and Junkins \cite[Ch.~5]{crassidis2004optimal} and obtain a correlated Kalman filter in Appendix~\ref{sec:proof_correlated_KF}.



\section{Distributed MSCKF}

In this section, we use Proposition \ref{prop:VI_Gaussian_sol} to derive a distributed version of the MSCKF algorithm \cite{msckf}, summarized in Algorithm~\ref{alg:distributed_KF}. Each step of the algorithm is described in the following subsections.

\begin{algorithm}
\caption{Distributed Multi-State Constraint Kalman Filter}
\label{alg:distributed_KF}
\begin{algorithmic}
    \State \textbf{Input:} Prior mean and covariance $(\boldsymbol{\mu}_{i,t-1}, \Sigma_{i,t-1})$, control input $\mathbf{u}_{i,t-1}$, and measurements $\mathbf{z}^{g}_{i,t}$, $\mathbf{z}^o_{i,t}$.
    \State \textbf{Output:} Posterior mean and covariance $(\boldsymbol{\mu}_{i,t}, \Sigma_{i,t})$

    \State 1: \textbf{Consensus averaging}: \eqref{eq:averaging_KF}, \eqref{eq:averaging_pose} in Sec.~\ref{sec:consensus_averaging}
    
    \State 2: \textbf{State propagation}: \eqref{eq:state_propagation_mean}, \eqref{eq:state_propagation_cov} in Sec.~\ref{sec:state_propagation}

    \State 3: \textbf{State update}: \eqref{eq:Kalman_matrix}, \eqref{eq:update_step} in Sec.~\ref{sec:msckf_ekf_update}
    
    \State 4: \textbf{Feature initialization}: \eqref{eq:landm_init} in Sec.~\ref{sec:feature_init}
    
    
\end{algorithmic}
\end{algorithm}

\vspace{-0.5cm}

\subsection{State and observation description}

The state $\bfs_{i,t}$ of robot $i$ at time $t$ contains a sequence of $c$ historical camera poses $\bfx_{i,t}$ and a set of $m_t$ landmarks $\bfy_{i,t}$:
\begin{equation}
	\begin{aligned}
	\bfs_{i,t} &= (\bfx_{i,t}, \bfy_{i,t} ), \\
	\bfx_{i,t} &= (T_{i,t-c+1}, \dots, T_{i,t}), \;\; T_{i,k} \in SE(3), \ \forall k, \\
	\bfy_{i,t} &= [\bfp^\top_{i,1} \ \dots \ \bfp^\top_{i,m_t}]^\top\!, \;\;\;\; \bfp_{i,k} \in \bbR^3, \ \forall k. 
	\end{aligned}
\end{equation}
Besides the joint mean of the historical camera poses $\bfx_{i,t}$ and landmarks $\bfy_{i,t}$, each robot also keeps track a joint covariance $\Sigma_{i,t} \in \bbR^{(6c+3m_t) \times (6c+3m_t)}$.
Each robot obtains observations $\bfz^o_{i,t}$ of persistent features, e.g., object detections, and observations $\bfz^g_{i,t}$ of opportunistic features, e.g., image keypoints or visual features, as illustrated in Fig.~\ref{fig:teaser}.
We use point observations and the pinhole camera model for both feature types. Only the landmarks associated with persistent features are initialized and stored in the state while the landmarks associated with opportunistic features are used for structureless updates as in the MSCKF algorithm \cite{msckf}.

\subsection{Consensus averaging}
\label{sec:consensus_averaging}
Each robot $i$ communicates with its neighbors $\calN_i$ to find out common landmarks. Then each robot sends the mean and covariance of the common landmarks $\boldsymbol{\mu}^{\bfy}_{i,t-1}, \Sigma^{\bfy}_{i,t-1}$ to its neighbors and receives $\boldsymbol{\mu}^{\bfy}_{j,t-1}, \Sigma^{\bfy}_{j,t-1}$, $j \in \calN_{i} \backslash \{i\}$. The consensus averaging step is carried out by averaging the marginal distributions of the common landmarks:
\begin{equation}
\begin{aligned}
    \bar{\Sigma}^{\bfy -1}_{i,t-1} &= \sum_{j \in \mathcal{N}_i} A_{ij} \Sigma^{\bfy -1}_{j,t-1}, \\
    \bar{\Sigma}^{\bfy -1}_{i,t-1} \bar{\boldsymbol{\mu}}^{\bfy}_{i,t-1} &= \sum_{j \in \mathcal{N}_i} A_{ij} \Sigma^{\bfy -1}_{j,t-1} \boldsymbol{\mu}^{\bfy}_{j,t-1},
\end{aligned}
\label{eq:averaging_KF}
\end{equation}
which is the same as \eqref{eq:MD_averaging} except that we only consider one time step $t$ here. Then, we need to reconstruct the new joint distribution $\calN(\bar{\boldsymbol{\mu}}_{i,t-1}, \bar{\Sigma}_{i,t-1})$. Since we store the historical camera poses as $SE(3)$ matrices, Proposition~\ref{prop:construct_joint} can not be applied directly. Following \cite[Ch.~7.3.1]{state_barfoot}, we define a Gaussian distribution over a historical camera pose $\underline{T}_{i,k},\ k=t-c, \cdots, t-1$ by adding a perturbation $\epsilon_{i,k}$:
\begin{equation}
    \underline{T}_{i,k} = T_{i,k} \exp(\epsilon_{i,k}^{\wedge}),\ \epsilon_{i,k} \sim \calN(\mathbf{0}_6, \Sigma^{\bfx_k}_{i,t-1}),
\end{equation}
where $(\cdot)^{\wedge}$ defined in \cite[Ch.~7.1.2]{state_barfoot} converts from $\bbR^{6}$ to a $\bbR^{4 \times 4}$ twist matrix. The estimated covariance $\Sigma_{i,t-1}$ already takes account of both the poses and landmarks. For consistency of notation with Proposition~\ref{prop:construct_joint}, we denote the mean of the pose perturbation as $\boldsymbol{\mu}^{\bfx}_{i,t-1} = \mathbf{0}_{6c}$. After averaging and reconstructing the new joint distribution, $\bar{\boldsymbol{\mu}}^{\bfx}_{i,t-1}$ may be non-zero, so we need to correct the camera poses as follows:
\begin{equation}
    \bar{T}_{i,k} \!=\! T_{i,k} \! \exp(\bar{\boldsymbol{\mu}}^{\bfx_k \wedge}_{i,t-1}), \ \bar{\boldsymbol{\mu}}^{\bfx_k}_{i,t-1} \! \in \! \bbR^{6}, k\!=\!t-c,\! \cdots\!,\! t-1.
    \label{eq:averaging_pose}
\end{equation}

\subsection{State propagation}
\label{sec:state_propagation}

We derive a general odometry propagation step for the MSCKF algorithm, thus not necessarily requiring IMU measurements and enabling vision-only propagation. We assume an odometry algorithm (e.g., libviso2~\cite{libviso2}) provides relative pose measurements $\delta T_{i,t-1}$ between the frame at time $t-1$ and that at time $t$. The state of robot $i$ is propagated as:
\begin{equation}
\begin{aligned}
\bfx^+_{i,t} &= ( \bar{T}_{i,t-c+1}, \ldots, \bar{T}_{i,t-1}, \bar{T}_{i,t-1} \delta T_{i,t-1} ), \\
\bfs^+_{i,t} &= ( \bfx^+_{i,t}, \bar{\boldsymbol{\mu}}^{\bfy}_{i,t-1} ),
\end{aligned}
\label{eq:state_propagation_mean}
\end{equation}
where the terms $\bar{(\cdot)}$ are obtained from the consensus averaging step. The state covariance is propagated as follows:
\begin{align}
\Sigma^+_{i,t} \! &= \!\!
\begin{bmatrix}
A & 0 \\
J_{t}    & 0 \\
0 & I_{3m_{t-1}}
\end{bmatrix} \!\! \bar{\Sigma}_{i,t-1} \!\!
\begin{bmatrix}
A & 0 \\
J_{t}    & 0 \\
0 & I_{3m_{t-1}}
\end{bmatrix}^\top \!\!\!\!\!\! + \operatorname{diag}(\bfe_{6n}) \otimes W_i, \notag \\
A  &= \!\!  \left[0_{6(c-1) \times 6} | I_{6(c-1)} \right]\!, J_{t} \!\! = \!\! \left[ 0_{6 \times 6(c-1)} \ \! Ad(\delta T_{t-1}^{-1}) \right]\!,\!\! \label{eq:state_propagation_cov}
\end{align}
where $Ad(\cdot)$ is the adjoint of an $SE(3)$ matrix~\cite[Chapter~7.1.4]{state_barfoot}, $\bfe_{6n} \in \bbR^{6n+3m_t}$ is a vector with the $6n$-th element as $1$ and the rest as $0$, and $W_i \in \bbR^{6\times6}$ is the odometry measurement covariance.




\subsection{State update}
\label{sec:msckf_ekf_update}
The MSCKF update step follows prior work \cite{smsckf}. The camera pose residual is a perturbation $\underline{\epsilon}_{i,k}$ that transforms the estimated pose $T_{i,k}$ to the true pose $\underline{T}_{i,k}$, i.e. $\underline{T}_{i,k} = T_{i,k} \exp(\underline{\epsilon}_{i,k}^{\wedge})$. The landmark residual is the difference between true position $\underline{\bfp}$ and the estimated position $\bfp$, i.e., $\tilde{\bfp} = \underline{\bfp} - \bfp$, where $\bfp$ is the position mean if the landmark is in the state or the result of triangulation if it is not. When robot $i$ receives the $k$-th geometric feature observation at time $t$, denoted as $\bfz^g_{i,t,k}$, we linearize the observation model around the current error state $\tilde{\bfs}_{i,t,k}$ (composed of both pose and landmark residuals) and the feature position residual $\tilde{\bfp}^{g}_{i,k}$:
\begin{equation*}
\bfr^g_{i,t,k} = \bfz^g_{i,t,k} - \hat{\bfz}^g_{i,t,k} = H^{\bfs, g}_{i,t,k} \tilde{\bfs}_{i,t,k} + H^{\bfp, g}_{i,t,k} \tilde{\bfp}^{g}_{i,k} + \bfv^{g}_{i,t,k},
\end{equation*}
where $\hat{\bfz}^g_{i,t,k}$ is the predicted observation, $H^{\bfs, g}_{i,t,k}$ and $H^{\bfp, g}_{i,t,k}$ are Jacobians, and $\bfv^g_{i,t,k} \sim \calN(\mathbf{0}, V^g_i)$ is the geometric observation noise. Then, we left-multiply by the nullspace $N_{i,t,k}$ of $H^{\bfp,g}_{i,t,k}$ to remove the effect of $\tilde{\bfp}^{g}_{i,k}$:
\begin{equation}
\bfr^{g,0}_{i,t,k} = N_{i,t,k}^\top \bfr^g_{i,t,k} = N_{i,t,k}^\top H^{\bfs,g}_{i,t,k} \tilde{\bfs}_{i,t,k} + N_{i,t,k}^\top \bfv^{g}_{i,t,k}.
\end{equation}
Since we allow general odometry in the propagation step, potentially obtained from visual features, there may be correlation between the motion noise and the observation noise.
The correlation is denoted as
\begin{equation}
    S_{i,t,k} = \bbE [\bfw_{i,t-1} \bfv_{i,t,k}^\top], \;\; \bfw_{i,t-1} \sim \calN(\mathbf{0}, W_i).
    \label{eq:correlation_def}
\end{equation}
Concatenating $\bfr^{g,0}_{i,t,k}$, $N_{i,t,k}^\top H^{\bfs,g}_{i,t,k}$, $N_{i,t,k}^\top V_i^g N_{i,t,k}$, and $S_{i,t,k} N_{i,t,k}$ for all $k$ appropriately, we get an overall geometric feature residual $\bfr^{g}_{i,t}$, Jacobian $H^{g}_{i,t}$, noise covariance $V^g_{i,t}$ and correlation $S_{i,t}$. Similarly, we linearize the observation model for the object $\bfz^{o}_{i,t,k}$:
\begin{equation*}
\bfr^{o}_{i,t,k} = \bfz^{o}_{i,t,k} - \hat{\bfz}^{o}_{i,t,k} = H^{\bfs, o}_{i,t,k} \tilde{\bfs}_{i,t} + \bfv^{o}_{i,t,k}, \ \bfv^{o}_{i,t,k} \sim \calN(\mathbf{0}, V^o_i),
\end{equation*}
and concatenate $\bfr^{o}_{i,t,k}$, $H^{\bfs, o}_{i,t,k}$ and $V_i^o$ for all $k$ to get an overall object observation residual $\bfr^{o}_{i,t}$, Jacobian $H^{o}_{i,t}$, and noise covariance $V^o_{i,t}$. 

Finally, by concatenating the residuals and Jacobians of both geometric and object features, we get the overall residual $\bfr_{i,t} \! = \! [\bfr^{g\top}_{i,t} \ \bfr^{o\top}_{i,t}]^\top$ and Jacobian $H_{i,t} \! = \! [H^{g\top}_{i,t} \ H^{o\top}_{i,t}]^\top$. As shown in Appendix~\ref{sec:proof_correlated_KF} and \cite[Ch.~5]{crassidis2004optimal}, the Kalman gain is:
\begin{align}
K_{i,t} &= (\Sigma^+_{i,t} H^{\top}_{i,t} + S_{i,t}) (H_{i,t} \Sigma^+_{i,t} H^{\top}_{i,t} +  V_{i,t})^{-1},
\label{eq:Kalman_matrix} \\
V_{i,t} &= \blkdiag(V^g_{i,t} + H^g_{i,t} S_{i,t} + S^\top_{i,t} H^{g\top}_{i,t}, V^o_{i,t}),
\end{align}
where $S_{i,t}$ only appears for the geometric features, $\Sigma^+_{i,t}$ is from the prediction step, and $K_{i,t}$ can be split to $K^{\bfx}_{i,t}$ and $K^{\bfy}_{i,t}$ related to $\bfx_{i,t}$ and $\bfy_{i,t}$ respectively. The landmark mean, camera poses, and the entire covariance are updated as:
\begin{align}
\boldsymbol{\mu}^{\bfy}_{i,t} &= \boldsymbol{\mu}^{\bfy+}_{i,t} + K^{\bfy}_{i,t} \bfr_{i,t}, \notag \\
T_{i,k} &= T^+_{i,k} \exp(K^{\bfx_k}_{i,t} \bfr_{i,t}), \ K^{\bfx_k}_{i,t} \bfr_{i,t} \in \bbR^{6}, \ k\!=\! t\!-\!c\!+\!1,  \cdots ,  t, \notag \\
\Sigma_{i,t} &= (I - K_{i,t} H_{i,t}) \Sigma^+_{i,t}
\label{eq:update_step}
\end{align}
where the terms $(\cdot)^+$ are from the propagation step.

\subsection{Feature initialization}
\label{sec:feature_init}
The feature initialization is the same as \cite{OpenVINS}. To initialize an object landmark, we first linearize the observation model
\begin{equation*}
\Tilde{\bfz}^{o}_{i,t,k} = 
H^{\bfs}_{i,t,k} \Tilde{\bfs}_{i,t} +  H^{\bfp}_{i,t,k} \Tilde{\bfp}^{o}_{i,k} + \bfv^{o}_{i,t,k}, \ \bfv^{o}_{i,t,k} \sim \calN(\mathbf{0}, V^o_{i}),
\end{equation*}
where $\Tilde{\bfz}^o_{i,t,k}$, $\Tilde{\bfs}_{i,t}$ and $\Tilde{\bfp}^o_{i,k}$ are the residuals of the observation, current state, and new landmark respectively. Then, QR decomposition is performed to separate the linearized observation model into two parts: one that depends on the new landmark and another that does not:
\begin{equation}
\begin{bmatrix} \Tilde{\bfz}^{o,1}_{i,t,k} \\ \Tilde{\bfz}^{o,2}_{i,t,k} \end{bmatrix} = 
\begin{bmatrix} H^{\bfs,1}_{i,t,k} & H^{\bfp,1}_{i,t,k} \\ H^{\bfs,2}_{i,t,k} & 0 \end{bmatrix}
\begin{bmatrix} \Tilde{\bfs}_{i,t} \\ \Tilde{\bfp}^{o}_{i,k} \end{bmatrix} + 
\begin{bmatrix}
\bfv^{o,1}_{i,t,k} \\ \bfv^{o,2}_{i,t,k}
\end{bmatrix}.
\end{equation}
Thus, we can augment the current state and covariance:
\begin{equation}
    \begin{aligned}
    \bfp^{o}_{i,k} &= \hat{\bfp}^{o}_{i,k} + H^{\bfs,1 -1}_{i,t,k}\Tilde{\bfz}^{o,1}_{i,t,k}, \\
    \Sigma^{\bfs\bfp}_{i,t,k} &= -\Sigma_{i,t} H^{\bfs,1 \top}_{i,t,k} H^{\bfp,1 -\top}_{i,t,k}, \\
    \Sigma^{\bfp}_{i,t,k} &= H^{\bfp,1 -1}_{i,t,k} (H^{\bfs,1}_{i,t,k} \Sigma_{i,t} H^{\bfs,1 \top}_{i,t,k} + V^{o,1}_{i}) H^{\bfp,1 -\top}_{i,t,k},
    \end{aligned}  
    \label{eq:landm_init}
\end{equation}
where $V^{o,1}_{i}$ is the covariance of noise $\bfv^{o,1}_{i,t,k}$, 
$\Sigma^{\bfs\bfp}_{i,t,k}$ is the cross-correlation term between the current state and new landmark, and $\Sigma^{\bfp}_{i,t,k}$ is the covariance of the new landmark.

\begin{figure*}[!ht]
    \centering
    \begin{subfigure}{0.68\linewidth}
        \includegraphics[width=\linewidth]{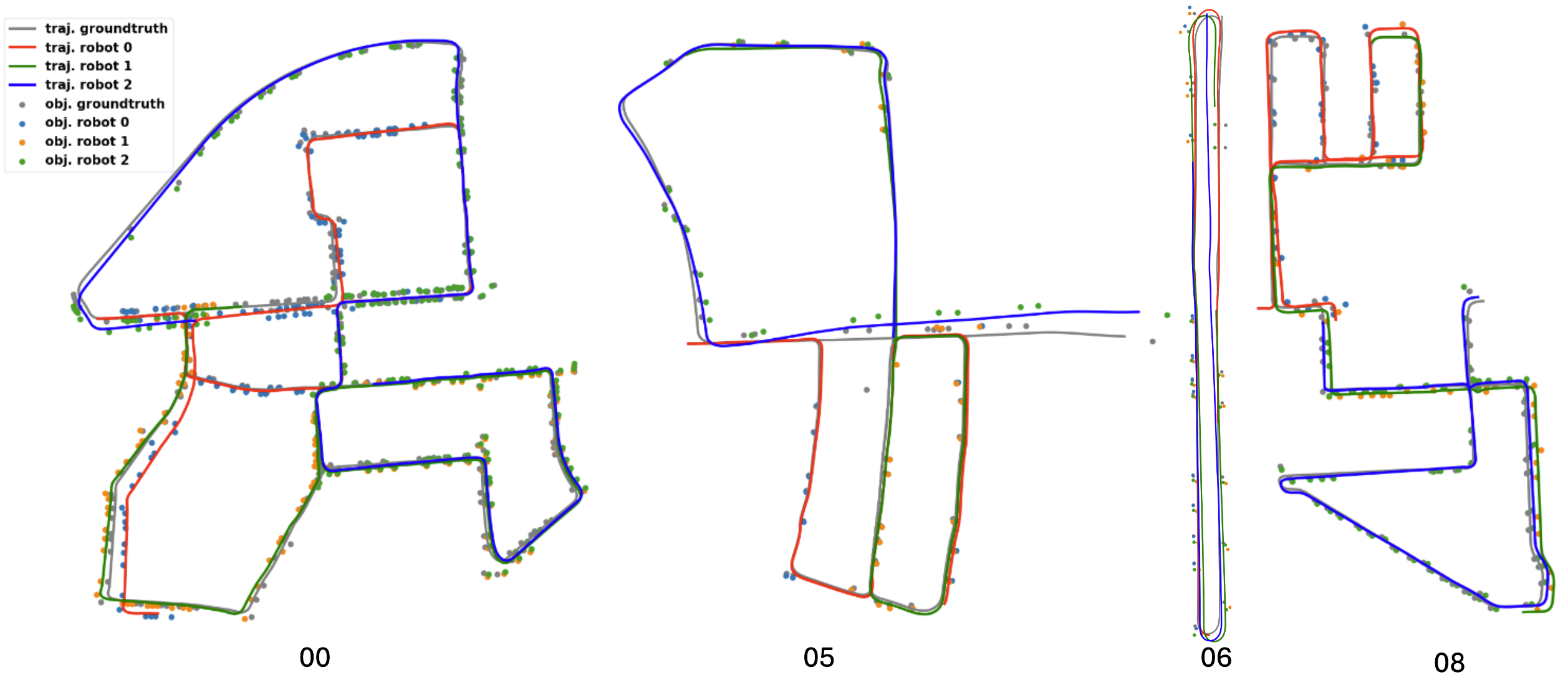}
        \caption{}
        \label{fig:kitti_visualization}
    \end{subfigure}%
    \hspace{0.1cm}%
    \begin{subfigure}{0.29\linewidth}
        \includegraphics[width=\linewidth]{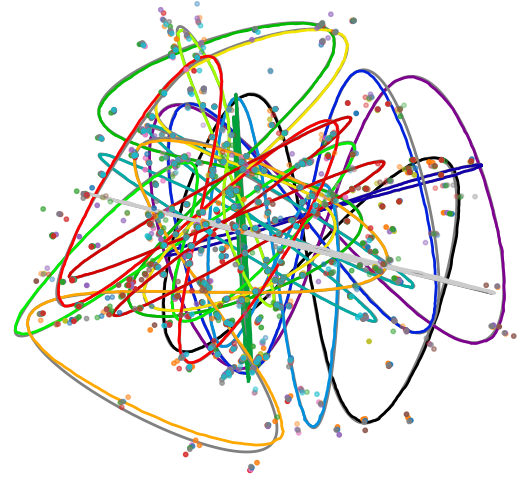}
        \caption{}
        \label{fig:simulation_visualization}
    \end{subfigure}
    \caption{Trajectory and object estimates of (a) 3 robots on KITTI sequences 00, 05, 06 and 08, (b) 15 robots in simulation.}
    \label{fig:qualitative_results}
\end{figure*}


\section{Evaluation}
\label{sec:evaluation}

We implemented the distributed MSCKF using only stereo camera observations and evaluated it on the KITTI dataset \cite{kitti} and on a simulated dataset with a larger number of robots. All experiments were carried out on a laptop with i9-11980HK@2.60 CPU, 16 GB RAM, and RTX 3080 GPU.

\subsection{KITTI dataset}

The KITTI dataset~\cite{kitti} is an autonomous driving dataset that provides stereo images, LiDAR point clouds, and annotated ground-truth robot trajectories. We provide details about the data processing and evaluation results below.

\subsubsection{Sequences and splits}
We chose long sequences in the KITTI odometry dataset with loop closures and a sufficient number of cars, used as object landmarks, namely, sequences 00, 05, 06, and 08. Each sequence is split into 3 sub-sequences representing 3 different robots. The sequence splits are as follows: sequence 00: $[0, 2000]$, $[1500, 3500]$, $[2500, 4540]$; sequence 05: $[0, 1200]$, $[800, 2000]$, $[1560, 2760]$; sequence 06: $[0, 700]$, $[200, 900]$, $[400, 1100]$; sequence 08: $[0, 2000]$, $[1000, 3000]$, $[2000, 4070]$. We used a fully connected graph and the adjacency matrix $A \in \mathbb{R}^{3\times3}$ has all elements as $\frac{1}{3}$.

\subsubsection{Geometric features}
%
We extract geometric features using the FAST corner detector \cite{fast_feature}. The KLT optical flow algorithm \cite{KLT} is used to track the features across stereo images. Outlier rejection is performed using 2-point RANSAC for temporal tracking and the known essential matrix for stereo matching. Finally, circular matching similar to \cite{circular_outlier_rejection} is performed to further remove outliers.

\subsubsection{Object features}
We utilize YOLOv6~\cite{yolov6} to detect object bounding boxes and compute the centers as our object observations. Since our work does not focus on object tracking, we directly use the instance ID annotations in SemanticKITTI \cite{semantickitti} for data association. The instance annotations are provided for LiDAR point clouds and we associate them with the bounding boxes by projecting the LiDAR point clouds onto the image plane and checking the dominant instance points inside each bounding box. 


\subsubsection{Odometry}
The relative pose $\delta T_{i,t}$ between consecutive camera frames is obtained by libviso2 \cite{libviso2}.

\begin{table}[!ht]
\caption{Trajectory RMSE in meters on KITTI sequences. Separate and consensus correspond to without/with the consensus averaging step in Sec.~\ref{sec:consensus_averaging}.}
\begin{tabular}{llllll}
\hline
                 & Robot 1 & Robot 2 & Robot 3 & Avg & Max \\ \hline
 libviso2~\cite{libviso2}  & 14.30   & 13.73  & 12.65  & 13.56 & 14.30  \\
00 Separate  & \bf{12.47}   & 7.55  & 12.42  & 10.81 & \bf{12.47}  \\
00 Consensus & 12.51   & \bf{7.13}  & \bf{8.73}   & \bf{9.45} & 12.51   \\ \hline
05 libviso2~\cite{libviso2}  & 5.36   & 6.42   & 11.57   & 7.78 & 11.57   \\
05 Separate  & 7.18   & 10.03   & \bf{7.87}   & 8.36 & 10.03   \\
05 Consensus & \bf{4.69}   & \bf{7.75}   & 9.56   & \bf{7.33} & \bf{9.56}   \\ \hline
06 libviso2~\cite{libviso2}  & 5.45   & 6.89   & 5.21   & 5.85 & 6.89   \\
06 Separate  & 4.23   & \bf{5.60}   & 4.86   & 4.90 & \bf{5.60}   \\
06 Consensus & \bf{4.23}   & 5.61   & \bf{4.76}   & \bf{4.87} & 5.61   \\ \hline
08 libviso2~\cite{libviso2}  & 9.17  & 21.05   & 11.37   & 13.86 & 21.05   \\
08 Separate  & 15.08  & 24.28   & 9.18   & 16.18 & 24.28   \\
08 Consensus & \bf{13.89}  & \bf{12.71}   & \bf{9.18}   & \bf{11.93} & \bf{13.89}   \\ \hline
\end{tabular}
\label{tab:kitti_traj_results}
\end{table}

\begin{table}[!ht]
\caption{Object estimation errors in meters on KITTI sequences. Separate and consensus correspond to without/with the consensus averaging step in Sec.~\ref{sec:consensus_averaging}.}
\begin{tabular}{llllll}
\hline
& Robot 1 & Robot 2 & Robot 3 & Avg & Max \\ \hline
00 Separate  & \bf{8.76}   & 7.61  & 8.70  & 8.36 & \bf{8.76}  \\
00 Consensus & 9.30   & \bf{6.74}  & \bf{7.16}   & \bf{7.73} & 9.30   \\ \hline
05 Separate  & 6.08   & 8.40   & 6.92   & 7.14 & \bf{8.40}   \\
05 Consensus & \bf{4.56}   & \bf{7.51}   & 8.54   & \bf{6.87} & 8.54   \\ \hline
06 Separate  & \bf{3.43}   & 5.92   & 4.64   & 4.66 & 5.92   \\
06 Consensus & 3.78   & \bf{5.63}   & \bf{4.37}   & \bf{4.59} & \bf{5.63}   \\ \hline
08 Separate  & 12.14  & 21.91   & \bf{8.19}   & 14.08 & 21.91   \\
08 Consensus & \bf{12.11}  & \bf{13.71}   & 9.21   & \bf{11.68} & \bf{13.71}   \\ \hline
\end{tabular}
\label{tab:kitti_obj_results}
\end{table}

\begin{table}[!ht]
\caption{Object position differences in meters across different robots on KITTI sequences. Separate and consensus correspond to without/with the consensus averaging step in Sec.~\ref{sec:consensus_averaging}.}
\begin{tabular}{llllll}
\hline
& Robot 1 & Robot 2 & Robot 3 & Avg & Max \\ \hline
00 Separate  & 9.69   & 10.35  & 8.35  & 9.46 & 10.35  \\
00 Consensus & \bf{5.95}   & \bf{8.62}  & \bf{5.11}   & \bf{6.56} & \bf{8.62}   \\ \hline
05 Separate  & 7.25   & 10.20   & 15.74   & 11.06 & 15.74   \\
05 Consensus & \bf{1.50}   & \bf{5.15}   & \bf{10.17}   & \bf{5.61} & \bf{10.17}   \\ \hline
06 Separate  & 5.56   & 4.97   & 4.79   & 5.12 & 5.56   \\
06 Consensus & \bf{5.01}   & \bf{4.68}   & \bf{4.43}   & \bf{4.71} & \bf{5.01}   \\ \hline
08 Separate  & 14.61  & 20.24   & 23.81   & 19.55 & 23.81   \\
08 Consensus & \bf{6.50}  & \bf{9.48}   & \bf{11.36}   & \bf{9.12} & \bf{11.36}   \\ \hline
\end{tabular}
\label{tab:kitti_obj_diff}
\end{table}

\begin{figure}[!ht]
    \centering
    \includegraphics[width=0.85\linewidth]{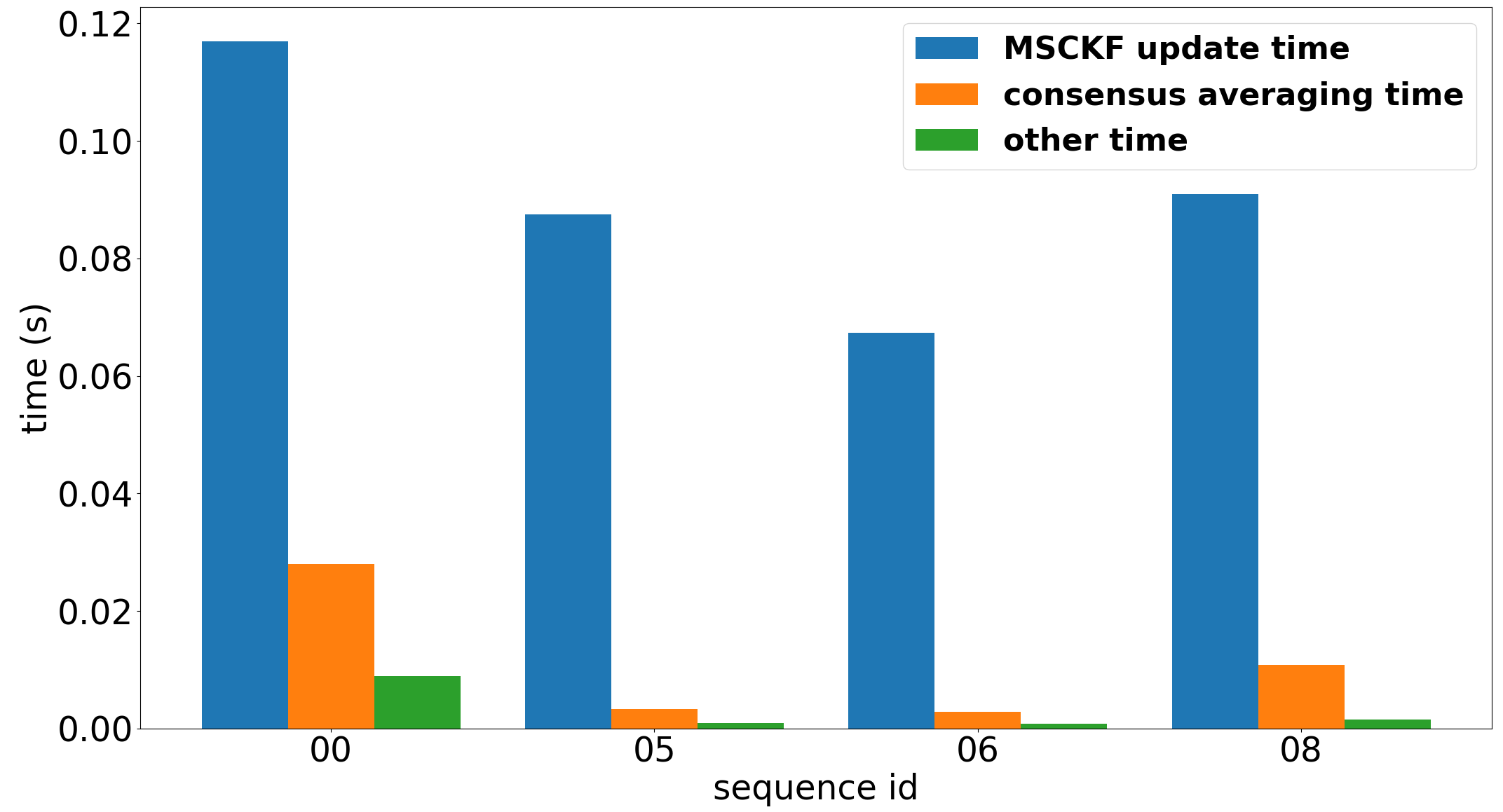}
    \caption{Time consumed by different components per robot per frame, including MSCKF update, consensus averaging, and other (prediction and landmark initialization).}
    \label{fig:time_analysis}
\end{figure}

\begin{figure}
    \centering
    \includegraphics[width=\linewidth]{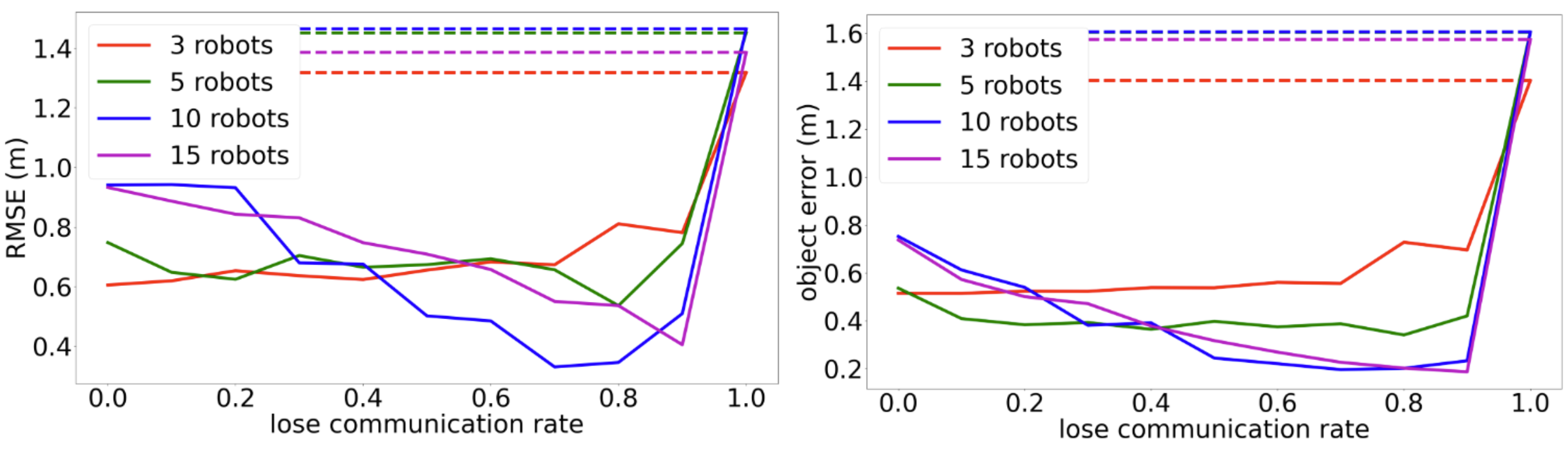}
    \caption{Analysis of the effect of the robot network connectivity. The dashed lines show RMSE without averaging.}
    \label{fig:connectivity_analysis}
\end{figure}

\subsubsection{Results and analysis}

We found empirically that setting the correlation matrix \eqref{eq:correlation_def} to zero gives the best results. We assume that this is because libviso2 \cite{libviso2} uses SURF features \cite{bay2006surf}, while the update step is performed using FAST features \cite{fast_feature} and the correlation is negligible. Qualitative results from three-robot collaborative object SLAM on the KITTI dataset are shown in Fig.~\ref{fig:kitti_visualization}. We show the root mean square error (RMSE) of the robot trajectory estimates in Table \ref{tab:kitti_traj_results} and the mean distances between estimated object positions and the ground truth in Table~\ref{tab:kitti_obj_results}. We do not use alignment for the trajectory RMSE \cite{rpg} because trajectory transformations affect the object mapping errors.
Some ways to mitigate the effect of bad estimates include resilient consensus \cite{Saldana2017resilient} or adaptive adjacency weights $A_{ij}$ depending on the robots' measurement accuracy.
Although consensus averaging can harm the estimation accuracy for some robots compared to running individual MSCKF algorithms for each robot, it helps improve the overall team performance in both localization and object mapping. The separate MSCKF sometimes perform worse than libviso2~\cite{libviso2}. This is because the object observation is too noisy. Updating with only object features can give an error up to 10 times as in Table~\ref{tab:kitti_traj_results}.
The distributed MSCKF achieves better agreement in the map estimates among the robots. We compare the object position differences with and without averaging in Table~\ref{tab:kitti_obj_diff} to quantify the reduction in disagreement. We also claim that the consensus averaging step does not add much time overhead because the robots communicate only common landmarks, meaning that the corresponding covariance $\Sigma^{\bfy}_{i,t}$ in \eqref{eq:averaging_KF} is small, and only perform averaging with one-hop neighbors. The computation time used by different components in the algorithm is shown in Fig.~\ref{fig:time_analysis}. Consensus averaging takes a small portion of time compared with the MSCKF update.

\subsection{Simulated data}
To test our algorithm with increasing numbers of robots, we generated simulated data for 3-15 robots.

\subsubsection{Data generation} 
Each robot moves along a Lissajous curve and odometry measurements are generated by adding perturbations to the relative transformation between consecutive poses. The landmarks, both geometric and objects, are generated by randomly sampling from Gaussian distributions centered at each trajectory point. There are 210 objects in the scene with different numbers of robots. The observations are then generated by projecting the corresponding landmarks onto the image plane and adding noise. 


\subsubsection{Results}

The results are visualized in Fig.~\ref{fig:simulation_visualization}. The quantitative results from the simulations with a fully connected graph are shown in Table~\ref{tab:simulation_results}. Our algorithm scales efficiently with an increasing number of robots while continuing to outperform decoupled MSCKF algorithms for each robot. We also analyze the effect of connectivity in Fig.~\ref{fig:connectivity_analysis}. In the experiment, each robot loses communication with each neighbor according to rate $r$, i.e., each edge in the graph at each time step is removed independently with probability $r$. We see that our algorithm is robust to communication loss rate up to $r = 0.9$. With small numbers of robots (3 and 5), the errors oscillate as the loss rate increases but with relatively large numbers of robots (10 and 15), the errors seem to decrease as the loss rate increases. Analyzing the effect of randomly connected graphs, e.g., broad gossip \cite{aysal2009broadcast}, will be considered in future work.

\begin{table}[t]
\centering
\caption{Trajectory, object errors in meters and consensus averaging time per robot per timestep in seconds in simulation with different numbers of robots. Separate and consensus correspond to without/with the consensus averaging step in Sec.~\ref{sec:consensus_averaging}.}
\begin{tabular}{lllll}
\hline
Number of robots        & 3 & 5 & 10 & 15 \\ \hline
Separate trajectory RMSE (m) & 1.318  & 1.452  & 1.464   & 1.385   \\
Consensus trajectory RMSE (m) & \bf{0.605}  & \bf{0.748}  & \bf{0.941}   & \bf{0.933}   \\ \hline
Separate object error (m) & 1.402  & 1.604  & 1.605   & 1.573   \\
Consensus object error (m) & \bf{0.514}  & \bf{0.536}  & \bf{0.752}   & \bf{0.737}   \\ \hline
Consensus averaging time (s) & 0.021 & 0.022 & 0.026 & 0.028 \\ \hline
\end{tabular}
\label{tab:simulation_results}
\end{table}

\section{CONCLUSION}
\label{sec:conclusion}

We developed a distributed vision-only filtering approach for multi-robot object SLAM. Our experiments demonstrate that the method improves both localization and mapping accuracy while achieving agreement among the robots on a common object map. Since the algorithm is fully distributed, it allows efficient scaling of the number of robots in the team. Having a common object map is useful for collaborative task planning, which we plan to explore in future work.


\section*{ACKNOWLEDGMENTS}
The authors are grateful to Shubham Kumar, Shrey Kansal, and Kishore Nukala from University of California San Diego for technical discussions and help with dataset preparation.

\bibliographystyle{ieeetr}
\bibliography{root}
\appendices
\appendix
\section{APPENDIX}
\label{sec:appendix}

\subsection{Proof of Proposition~\ref{prop:DM_solution}}
\label{sec:proof_DM_solution}
Taking the constraint $\int q_i = 1$ into account, we consider the Lagrangian:
\begin{equation*}
	\begin{aligned}
	\calL(q_i, \lambda) &= \bbE_{q_i}\brl{ \frac{\delta c_i}{\delta q_i} (q_i^{(k)})} + \frac{1}{\alpha^{(k)}} \KL(q_i(\bfx_i|\bfy) || q_i^{(k)}(\bfx_i|\bfy)) \\
	&\quad + \sum_{j \in \calN_i} \frac{A_{ij}}{\alpha_{k}} \KL(q_i(\bfy) || q_j^{(k)}(\bfy)) + \lambda (\int q_i - 1),
	\end{aligned}
\end{equation*}
where $\lambda$ is a multiplier. The variation of $\calL$ w.r.t. $q_i$ is:
\begin{equation*}
	\begin{aligned}
	\frac{\delta \calL}{\delta q_i} &= \frac{\delta c_i}{\delta q_i} (q_i^{(k)}) + \frac{1}{\alpha_{k}} (1 + \log q_i) - \frac{1}{\alpha_{k}} \log q_i^{(k)} (\bfx_i | \bfy) \\
	&\quad - \frac{1}{\alpha_{k}} \sum_{j \in \calN_i} A_{ij} \log q_j^{(k)} (\bfy) + \lambda \\
	&= -\log p_i(\bfx_i, \bfy, \bfz_i, \bfu_i) + 1 + \log q_i^{(k)} + \frac{1}{\alpha_{k}} (1 + \log q_i) \\
	&\quad - \frac{1}{\alpha_{k}} \log q_i^{(k)} (\bfx | \bfy) - \frac{1}{\alpha_{k}} \sum_{j \in \calN_i} A_{ij} \log q_j^{(k)} (\bfy) + \lambda. \\
	\end{aligned}
\end{equation*}
Setting the variation to zero and solving for $q_i$, leads to:
\begin{align*}
q_i &= e^{-1 - \alpha_{k} - \alpha_{k}\lambda} [p_i (\bfx_i, \bfy, \bfz_i, \bfu_i) / q_i^{(k)}(\bfx_i, \bfy)]^{\alpha_{k}} \\
&\quad \ \quad q_i^{(k)}(\bfx_i | \bfy) \prod_{j \in \calN_{i}} [q^{(k)}_j (\bfy)]^{A_{ij}} \qedhere
\\
&\propto [p_i (\bfx_i,\! \bfy,\! \bfz_i,\! \bfu_i) \! / \! q_i^{(k)}(\bfx_i,\! \bfy)]^{\alpha_{k}} q_i^{(k)}(\bfx_i | \bfy) \!\! \prod_{j \in \calN_{i}} \!\! [q^{(k)}_j (\bfy)]^{A_{ij}}.  \qedhere
\end{align*}
thus completing the proof.

\subsection{Proof of Proposition~\ref{prop:VI_Gaussian_sol}}
\label{sec:proof_VI_solution}

Suppose that the density $q_i^{(k)}(\bfx_i, \bfy)$ of $\calN(\boldsymbol{\mu}_{i,(k)}, \Sigma_{i,(k)})$ is the joint density at iteration $k$ and the density $q_j^{(k)}(\bfy)$ of $\calN(\boldsymbol{\mu}_{j,(k)}^{\bfy}, \Sigma_{j,(k)}^{\bfy})$ is the marginal density over $\bfy$. To save space, we simplify the notation $q_i^{(k)}(\bfx_i, \bfy)$ as $q_i^{(k)}$. The objective function~\eqref{eq:distributed_MD} can be rewritten as
\begin{align}
    &g_i^{(k)} = \bbE_{q_i} [\frac{\delta c_i}{\delta q_i} (q_i^{(k)})] + \frac{1}{\alpha_{k}} \bbE_{q_i}[\log q_i] - \frac{1}{\alpha_{k}} \bbE_{q_i} [\log q_i^{(k)}] \notag \\
    &\quad - \frac{1}{\alpha_{k}} \bbE_{q_i}[\log q_i(\bfy)] + \frac{1}{\alpha_{k}} \bbE_{q_i}[\log q^{(k)}_i(\bfy)] \notag \\
    &\quad + \sum_{j \in \calN_i} \frac{A_{ij}}{\alpha_{k}} \bbE_{q_i}[\log q_i(\bfy)] - \sum_{j \in \calN_i} \frac{A_{ij}}{\alpha_{k}} \bbE_{q_i} [\log q_j^{(k)} (\bfy)] \notag \\
    &= \bbE_{q_i} [\frac{\delta c_i}{\delta q_i} (q_i^{(k)})] - \frac{1}{2\alpha_{k}} \log |\Sigma_i| - \frac{1}{\alpha_{k}} \bbE_{q_i} [\log \frac{q_i^{(k)}}{q^{(k)}_i(\bfy)} ] \notag \\
    &\quad - \sum_{j \in \calN_i} \frac{A_{ij}}{\alpha_{k}} \bbE_{q_i} [\log q_j^{(k)} (\bfy)].
\end{align}
With the above objective, we can compute the derivatives with respect to $\boldsymbol{\mu}_i$ and $\Sigma_i$ as follows
\begin{align}
    &\frac{\partial g_i^{(k)}}{\partial \boldsymbol{\mu}^\top_i} = \Sigma^{-1}_i \bbE_{q_i} [(\bfs_i - \boldsymbol{\mu}_i) \frac{\delta c_i}{\delta q_i} (q_i^{(k)})] \notag \\
    &\quad - \frac{1}{\alpha_{k}} \Sigma_i^{-1} \bbE_{q_i} [(\bfs_i - \boldsymbol{\mu}_i) \log \frac{q_i^{(k)}}{q_i^{(k)}(\bfy)} ] \notag \\
    &\quad - \sum_{j \in \calN_i} \frac{A_{ij}}{\alpha_{k}} \Sigma_i^{-1} \bbE_{q_i}[(\bfs_i - \boldsymbol{\mu}_i) \log q_j^{(k)} (\bfy)], 
\label{eq:dmu}
\end{align}
\begin{align*}
    &\frac{\partial g_i^{(k)}}{\partial \boldsymbol{\mu}^\top_i \partial \boldsymbol{\mu}_i} = \Sigma_i^{-1} \bbE_{q_i} [(\bfs_i - \boldsymbol{\mu}_i)(\bfs_i - \boldsymbol{\mu}_i)^\top \frac{\delta c_i}{\delta q_i} (q_i^{(k)})] \Sigma_i^{-1} \notag \\
    &\ - \frac{1}{\alpha_{k}} \Sigma_i^{-1} \bbE_{q_i} [(\bfs_i - \boldsymbol{\mu}_i)(\bfs_i - \boldsymbol{\mu}_i)^\top \log \frac{q_i^{(k)}}{q_i^{(k)}(\bfy)} ] \Sigma_{i}^{-1} \notag \\
    &\ - \sum_{j \in \calN_i} \frac{A_{ij}}{\alpha_{k}} \Sigma_i^{-1} \bbE_{q_i} [(\bfs_i - \boldsymbol{\mu}_i)(\bfs_i - \boldsymbol{\mu}_i)^\top \log q_j^{(k)}(\bfy)] \Sigma_i^{-1}, \\
\end{align*}
\begin{align*}
    &\frac{\partial g_i^{(k)}}{\partial \Sigma_i} = \frac{1}{2} \Sigma_i^{-1} \bbE_{q_i} [(\bfs_i - \boldsymbol{\mu}_i)(\bfs_i - \boldsymbol{\mu}_i)^\top \frac{\delta c_i}{\delta q_i} (q_i^{(k)})] \Sigma_i^{-1} \\
    &\ - \frac{1}{2\alpha_{k}} \!\! \left( \Sigma_i^{-1} \! + \! \Sigma_i^{-1} \bbE_{q_i} [(\bfs_i \! - \! \boldsymbol{\mu}_i)(\bfs_i \! - \! \boldsymbol{\mu}_i)^\top \!\! \log \frac{q_i^{(k)}}{q_i^{(k)}(\bfy)} ] \Sigma_i^{-1} \right) \\
    & - \sum_{j \in \calN_i} \frac{A_{ij}}{2\alpha_{k}} \Sigma_i^{-1} \bbE_{q_i} [(\bfs_i - \boldsymbol{\mu}_i)(\bfs_i - \boldsymbol{\mu}_i)^\top \log q_j^{(k)}(\bfy)] \Sigma_i^{-1}.
\end{align*}
We can see that, 
\begin{equation}
    \frac{\partial g_i^{(k)}}{\partial \boldsymbol{\mu}^\top_i \partial \boldsymbol{\mu}_i} =  2 \frac{\partial g_i^{(k)}}{\partial \Sigma_i} + \frac{1}{\alpha_{k}} \Sigma_i^{-1}
    \label{eq:relation_secdmu_dSigma}
\end{equation}
To compute $\frac{\partial g_i^{(k)}}{\partial \boldsymbol{\mu}^\top_i \partial \boldsymbol{\mu}_i}$, following \cite{barfoot2020exactly}, we make use of \emph{Stein's lemma}~\cite{stein1981estimation}:
\begin{equation}
    \bbE_{q_i} [(\bfs_i - \boldsymbol{\mu}_i) f_i(\bfs_i)] \equiv \Sigma_i \bbE_{q_i} [\frac{\partial f_i(\bfs_i)}{\partial \bfs_i^\top}].
\end{equation}
Note that $\frac{\delta c_i}{\delta q_i} (q_i^{(k)})$, $\log q_i^{(k)}$ and $\log q_i^{(k)}(\bfy)$ are all functions of $\bfs_i$. We have $\frac{\delta c_i}{\delta q_i} = - \log p_i(\bfs_i, \bfz_i, \bfu_i) + \log q_i(\bfs_i) + 1$~\cite[Proposition~1]{paritosh2020marginal} and $\log p_i(\bfs_i, \bfz_i, \bfu_i)$ can be decomposed with \eqref{eq:likelyhood_decomposition} so that \eqref{eq:dmu} can be rewritten as
\begin{align}
    &\frac{\partial g_i^{(k)}}{\partial \boldsymbol{\mu}^\top_i} = \bbE_{q_i} [\frac{\partial \frac{\delta c_i}{\delta q_i} (q_i^{(k)})}{\partial \bfs^\top_i}] - \frac{1}{\alpha_{k}} \bbE_{q_i} [\frac{\partial \log q_i^{(k)}}{\partial \bfs^\top_i} ] \notag \\
    &\quad + \frac{1}{\alpha_{k}} \bbE_{q_i} [\frac{\partial \log q_i^{(k)}(\bfy)}{\partial \bfs^\top_i} ] - \sum_{j \in \calN_i} \frac{A_{ij}}{\alpha_{k}} \bbE_{q_i} [\frac{\partial \log q_j^{(k)}(\bfy)}{\partial \bfs^\top_i}] \notag \\
    &= - \Bar{F}^{-\top}_i \Bar{W}_i^{-1} (\Bar{G}_i \bfu_i - \Bar{F}^{-1}_i \boldsymbol{\mu}_i) - \Bar{H}^\top_i \Bar{V}_i^{-1} (\bfz_i - \Bar{H}_i \boldsymbol{\mu}_i) \notag \\
    &\quad - \Sigma_{i,(k)}^{-1} (\boldsymbol{\mu}_i - \boldsymbol{\mu}_{i,(k)}) + \frac{1}{\alpha_{k}} \Sigma_{i,(k)}^{-1} (\boldsymbol{\mu}_i - \boldsymbol{\mu}_{i,(k)}) \notag \\
    &\quad - \frac{1}{\alpha_{k}} L^\top \Sigma^{\bfy -1}_{i, (k)} L (\boldsymbol{\mu}_i - \boldsymbol{\mu}_{i,(k)}) \notag \\
    &\quad + \frac{1}{\alpha_{k}} \sum_{j \in \calN_i} A_{ij} L^\top \Sigma^{\bfy -1}_{j, (k)} L (\boldsymbol{\mu}_i - \boldsymbol{\mu}_{j,(k)}).
    \label{eq:dmu_stein}
\end{align}
where $\boldsymbol{\mu}_{i,(k)}, \ \Sigma_{i,(k)}$ are the mean and covariance for $q^{(k)}$ respectively and $L = [0 \ I]$ such that 
\begin{equation}
    L^\top \Sigma^{\bfy}_{i} L = \begin{bmatrix}
        0 & 0 \\
        0 & \Sigma^{\bfy}_{i} \\
    \end{bmatrix}.
\end{equation}
Therefore, the second derivative with respect to $\boldsymbol{\mu}_i$ is
\begin{equation}
    \begin{aligned}
        &\frac{\partial g_i^{(k)}}{\partial \boldsymbol{\mu}^\top_i \partial \boldsymbol{\mu}_i} = \Bar{F}^{-\top}_i \Bar{W}_i^{-1} \Bar{F}^{-1}_i + \Bar{H}^\top_i \Bar{V}_i^{-1} \Bar{H}_i - \Sigma_{i,(k)}^{-1} \\
        &\quad + \frac{1}{\alpha_{k}} (\Sigma_{i, (k)}^{-1} - L^\top \Sigma^{\bfy -1}_{i, (k)} L + \sum_{j \in \calN_{i}} A_{ij} L^\top \Sigma^{\bfy -1}_{j, (k)} L).
    \end{aligned}
    \label{eq:seond_dmu}
\end{equation}
Setting $\frac{\partial g_i^{(k)}}{\partial \Sigma_i} = 0$ and using \eqref{eq:relation_secdmu_dSigma}\eqref{eq:seond_dmu}, we get
\begin{equation}
    \begin{aligned}
        \Sigma_i^{-1} &= \alpha_{k} (\Bar{F}^{-\top}_i \Bar{W}_i^{-1} \Bar{F}^{-1}_i + \Bar{H}^\top_i \Bar{V}_i^{-1} \Bar{H}_i - \Sigma_{i,(k)}^{-1} ) \\
        &\ + \Sigma_{i,(k)}^{-1} - L^\top \Sigma^{\bfy -1}_{i, (k)} L + \sum_{j \in \calN_{i}} A_{ij} L^\top \Sigma^{\bfy -1}_{j, (k)} L.
    \end{aligned}
    \label{eq:info_matrix_intermediate}
\end{equation}
Further setting $\frac{\partial g_i^{(k)}}{\partial \boldsymbol{\mu}_i^\top} = 0$ and using \eqref{eq:dmu_stein}\eqref{eq:info_matrix_intermediate}, we get
\begin{equation}
    \begin{aligned}
        &\Sigma_i^{-1} \boldsymbol{\mu}_i = \alpha_{k} (\Bar{F}_i^{-\top} \Bar{W}_i^{-1} \Bar{G}_i \bfu_i + \Bar{H}_i^\top \Bar{V}_i^{-1} \bfz_i - \Sigma_{i,(k)}^{-1} \boldsymbol{\mu}_{i,(k)}) \\
        &\ + \Sigma_{i,(k)}^{-1} \boldsymbol{\mu}_{i,(k)} \!\! - \!\! L^\top \Sigma_{i, (k)}^{\bfy -1} L \boldsymbol{\mu}_{i,(k)} \!\! + \!\! \sum_{j \in \calN_i} A_{ij} L^\top \Sigma^{\bfy -1}_{j, (k)} L \boldsymbol{\mu}_{j,(k)}.
    \end{aligned}
    \label{eq:info_vector_intermediate}
\end{equation}
Note that $\calG(\Sigma^{\bfy -1}_{i,(k)}\boldsymbol{\mu}^{\bfy}_{i,(k)},  \Sigma^{\bfy -1}_{i,(k)})$ is the marginal distribution in the information space. In this way
\begin{equation}
    \begin{aligned}
        \sum_{j \in \calN_i} A_{ij} L^\top \Sigma^{\bfy -1}_{j, (k)} L
        = \begin{bmatrix}
            0 & 0 \\
            0 & \sum_{j \in \calN_i} A_{ij} \Sigma^{\bfy -1}_{j, (k)}
          \end{bmatrix}
        = \begin{bmatrix}
            0 & 0 \\
            0 & \bar{\Sigma}^{\bfy-1}_{i,(k)}
          \end{bmatrix}.
    \end{aligned}
\label{eq:marginal_cov_avg}
\end{equation}
Similarly,
\begin{align}
    \sum_{j \in \calN_i} A_{ij} L^\top \Sigma^{\bfy -1}_{j, (k)} L \boldsymbol{\mu}_{j,(k)} &= 
    \begin{bmatrix}
        \mathbf{0} \\
        \sum_{j \in \calN_i} \Sigma^{\bfy -1}_{j,(k)} \boldsymbol{\mu}_{j,(k)}
    \end{bmatrix} \notag \\ 
    &=
    \begin{bmatrix}
        \mathbf{0} \\
        \bar{\Sigma}^{\bfy -1}_{i,(k)} 
        \bar{\boldsymbol{\mu}}_{i,(k)}
    \end{bmatrix}. \label{eq:marginal_mean_avg}
\end{align}
Finally, setting
\begin{align}
\bar{\Omega}_{i,(k)} &= \Sigma_{i,(k)}^{-1} - L^\top \Sigma^{\bfy -1}_{i, (k)} L + \sum_{j \in \calN_{i}} A_{ij} L^\top \Sigma^{\bfy -1}_{j, (k)} L, \notag \\
\bar{\omega}_{i,(k)} &= \Sigma_{i,(k)}^{-1} \boldsymbol{\mu}_{i,(k)} - L^\top \Sigma_{i, (k)}^{\bfy -1} L \boldsymbol{\mu}_{i,(k)} \notag \\
&\quad + \sum_{j \in \calN_i} A_{ij} L^\top \Sigma^{\bfy -1}_{j, (k)} L \boldsymbol{\mu}_{j,(k)},  \label{eq:avg_info_pdf}
\end{align}
we can verify that
\begin{equation}
    \begin{aligned}
    \bar{\Omega}^{-1}_{i,(k)} &= \bar{\Sigma}_{i,(k)} = 
    \begin{bmatrix}
	A \bar{\Sigma}^{\bfy}_{i,(k)} A^\top + P & A \bar{\Sigma}^{\bfy}_{i,(k)} \\
	\bar{\Sigma}^{\bfy}_{i,(k)} A^\top & \bar{\Sigma}^{\bfy}_{i,(k)}
    \end{bmatrix} \\
    \bar{\Sigma}_{i,(k)} \bar{\omega}_{i,(k)} &= \bar{\boldsymbol{\mu}}_{i,(k)} = 
    \begin{bmatrix}
        A \bar{\boldsymbol{\mu}}^{\bfy}_{i,(k)} + \bfb \\
        \bar{\boldsymbol{\mu}}^{\bfy}_{i,(k)}
    \end{bmatrix},    
    \end{aligned}
\label{eq:new_joint}
\end{equation}
where
\begin{align}
    &A = \Sigma^{\bfx \bfy}_{i,(k)} \Sigma^{\bfy -1}_{i,(k)}, \bfb = \boldsymbol{\mu}^{\bfx}_{i,(k)} - \Sigma^{\bfx \bfy}_{i,(k)} \Sigma^{\bfy -1}_{i,(k)} \boldsymbol{\mu}^{\bfy}_{i,(k)}, \notag \\ 
    &P = \Sigma^{\bfx}_{i,(k)} - \Sigma^{\bfx}_{i,(k)} \Sigma^{\bfy -1}_{i,(k)} \Sigma^{\bfx \bfy \top}_{i,(k)}.
\end{align}
By substituting $\bar{\Omega}_{i,(k)} = \bar{\Sigma}_{i,(k)}^{-1}$ and $\bar{\omega}_{i,(k)} = \bar{\Sigma}_{i,(k)}^{-1} \bar{\boldsymbol{\mu}}_{i,(k)}$ into \eqref{eq:avg_info_pdf} and further into \eqref{eq:info_matrix_intermediate}\eqref{eq:info_vector_intermediate}, we will get exactly the same result as in Proposition~\ref{prop:VI_Gaussian_sol}~\eqref{eq:VI_Gaussian_update}, thus completing the proof.

The result can also be interpreted by Proposition~\ref{prop:DM_solution}, \eqref{eq:marginal_cov_avg} and \eqref{eq:marginal_mean_avg} are geometric averaging of the marginal Gaussians over $\bfy$~\cite[Lemma~1]{atanasov2014joint} which correspond to $\prod_{j \in \calN_{i}} [ q^{(k)}_j (\bfy) ]^{A_{ij}}$  in \eqref{eq:DM_solution}. \eqref{eq:new_joint} is constructing a new joint distribution with the averaged marginal (Proposition~\ref{prop:construct_joint}) which correspond to $ q_i^{(k)}(\bfx_i | \bfy) \prod_{j \in \calN_{i}} [ q^{(k)}_j (\bfy) ]^{A_{ij}}$ in \eqref{eq:DM_solution}. The step term $\alpha_k (\Bar{F}^{-\top}_i \Bar{W}_i^{-1} \Bar{F}^{-1}_i \!+ \Bar{H}^\top_i \Bar{V}_i^{-1} \Bar{H}_i -\! \Sigma^{-1}_{i,(k)})$ and $\alpha_{k} (\Bar{F}_i^{-\top} \Bar{W}_i^{-1} \Bar{G}_i \bfu_i + \Bar{H}_i^\top \Bar{V}_i^{-1} \bfz_i - \Sigma^{-1}_{i,(k)}\boldsymbol{\mu}_{i,(k)})$ correspond to $[p_i (\bfx_i, \bfy, \bfz_i, \bfu_i) / q_i^{(k)}(\bfx_i, \bfy)]^{\alpha_{k}}$ in \eqref{eq:DM_solution}.



\subsection{Derivation of the Kalman filter with noise correlation}
\label{sec:proof_correlated_KF}
Suppose that there is correlation between the motion and observation noise:
\begin{equation}
    S = \bbE[\bfw_{t-1} \bfv_t^\top],
\end{equation}
and take two timesteps of only one robot's state into account, namely, $\bfs_{t-1}$ and $\bfs_t$, by defining the following lifted terms
\begin{align}
    & \bfs \! = \! [\bfs_{t-1}^\top \ \bfs^\top_{t}]^\top\!\!, \bfu^\top \!\! = \! [\boldsymbol{\mu}_{t-1}^\top \ \bfu_{t-1}^\top ]^\top\!\!, \Bar{H} \! = \! [0 \ H], \bar{S} \! = \! [0 \ S^\top]^\top\!\!, \notag \\
    & \Bar{W} = 
    \begin{bmatrix}
        \Sigma_{t-1} & 0 \\
        0                   &  W
    \end{bmatrix}, 
    \Bar{F} = 
    \begin{bmatrix}
        I       & 0 \\
        F       & I \\
    \end{bmatrix}, 
    \Bar{G} = 
    \begin{bmatrix}
    I   & 0 \\
    0   & G
    \end{bmatrix}, \label{eq:lifted_terms}
\end{align}
where we omit the subscript $i$ since there is only one robot. The term $\log p(\bfs, \bfz_t, \bfu)$ in \eqref{eq:original_objective} can be re-written as:
\begin{equation}
    \log \! p(\bfs, \bfz_t, \bfu) \!\! \propto \! - \frac{1}{2} \!
    \begin{bmatrix}
        \bar{F}^{-1} \bfs \! - \! \bar{G} \bfu \\
        \bfz_t \! - \! \bar{H}\bfs
    \end{bmatrix} \!\!
    \begin{bmatrix}
        \bar{W} & \bar{S} \\
        \bar{S}^\top & V
    \end{bmatrix}^{-1} \!\!
    \begin{bmatrix}
        \bar{F}^{-1} \bfs \! - \! \bar{G} \bfu \\
        \bfz_t \! - \! \bar{H}\bfs
    \end{bmatrix}^\top \!\!\!\!.
\end{equation}
The derivatives in \eqref{eq:dmu}, i.e., $\frac{\partial g_i^{(k)}}{\partial \boldsymbol{\mu}^\top_i}$, $\frac{\partial g_i^{(k)}}{\partial \Sigma_i}$, and $\frac{\partial g_i^{(k)}}{\partial \Sigma_i}$ remain the same format except $c_i(q_i)$ in \eqref{eq:original_objective} changes. Therefore, we can follow the exact derivation steps in Sec.~\ref{sec:proof_VI_solution}, which we omit and directly show that \eqref{eq:regular_KF_info_space} becomes
\begin{align}\label{eq:correlated_KF_intermediate}
    \bar{\Sigma}^{-1} &= \bar{F}^{-\top} A^{-1} \bar{F} + \bar{H}^\top C^\top \bar{F}^{-1} + \bar{F}^{-\top} C \bar{H} + \bar{H} B^{-1} \bar{H}, \notag \\
    \bar{\Sigma}^{-1} \bar{\boldsymbol{\mu}} &= \bar{F}^{-\top} \bar{A}^{-1} \bar{G} \bar{\bfu} + \bar{H} \bar{C}^\top \bar{\bfu} + \bar{F}^{-\top} \bar{C} \bfz_t + \bar{H}^\top \bar{B}^{-1} \bfz_t,
\end{align}
where 
\begin{equation} \label{eq:intermediate_ABC}
    \bar{A} = \bar{W} - \bar{S} V^{-1} \bar{S}^\top, \ \bar{B} = V - \bar{S}^\top \bar{W}^{-1} \bar{S}, \bar{C} = -\bar{A}^{-1} \bar{S} V^{-1}.
\end{equation}
By substituting \eqref{eq:intermediate_ABC} and \eqref{eq:lifted_terms} into \eqref{eq:correlated_KF_intermediate}, we get
\begin{align}
    &\bar{\Sigma}^{-1} = \notag \\
    &\begin{bmatrix}
        \Sigma_{t-1}^{-1} \! + \! F^\top \!\! A^{-1} F & F^\top \! A^{-1} S V^{-1} H \! - \! F^\top \!\! A^{-1} \\
        H^\top \! V^{-1} \! S^\top \! A^{-1} \! F \! - \! A^{-1} \! F & A^{-1}\! + \! H^\top \!\! B^{-1} H ...  \\
        \ & - \! H^\top \! V^{-1} \! S^\top \!\! A^{-1} \! - \! A^{-1} S V^{-1} H
    \end{bmatrix}\!\!, \notag \\
    &\bar{\Sigma}^{-1} \bar{\boldsymbol{\mu}} = \notag \\
    &\begin{bmatrix}
        \Sigma_{t-1}^{-1} \boldsymbol{\mu}_{t-1} \! - \!F^\top A^{-1} G\bfu_t \! + \! F^\top A^{-1} S V^{-1} \bfz_t \\
        A^{-1} G \bfu_t \! - \! H^\top V^{-1} S^\top A^{-1} G \bfu_t \! - \! A^{-1} S V^{-1} \bfz_t \! + \! H^\top B^{-1} \bfz_t
    \end{bmatrix}\!\!.
    \label{eq:correlated_KF_intermediate_expanded}
\end{align}
where 
\begin{equation}
    A = W - S V^{-1} S^\top, \ B = V - S^\top W^{-1} S.
\end{equation}
Following the approach in Eq. 3.113 in \cite[Chapter~3.3.2]{state_barfoot}, we multiply both sides of \eqref{eq:correlated_KF_intermediate_expanded} by the matrix:
\begin{equation}
    \begin{bmatrix}
    I & 0 \\
    (A^{-1} F - H^\top V^{-1} S^\top A^{-1} F) (\Sigma_{t-1}^{-1} + F^\top A^{-1} F)^{-1} & I
    \end{bmatrix}
\end{equation}
to obtain:
\begin{align}
    &\quad \begin{bmatrix}
        \Sigma_{t-1}^{-1} \! + \! F^\top \! A^{-1} F & F^\top \! A^{-1} S V^{-1} H \! - \! F^\top \! A^{-1} \\
        0 & \Sigma_t^{-1}
    \end{bmatrix}
    \begin{bmatrix}
        \hat{\boldsymbol{\mu}}_{t-1} \\
        \boldsymbol{\mu}_t
    \end{bmatrix} \notag \\
    &=
    \begin{bmatrix}
        \Sigma_{t-1}^{-1} \boldsymbol{\mu}_{t-1} - F^\top A^{-1} G \bfu_t + F^\top S V^{-1} \bfz_t \\
        \Sigma_t^{-1} \boldsymbol{\mu}_t
    \end{bmatrix},
\end{align}
where $\hat{\boldsymbol{\mu}}_{t-1}$ is different from the estimation $\boldsymbol{\mu}_{t-1}$ at time $t-1$ since it is estimated with future observation, i.e. $\bfz_t$ and
\begin{align}
    \Sigma_t^{-1} \! =& \! - \! (H^\top V^{-1} S^\top A^{-1} F - A^{-1} F) (\Sigma_{t-1}^{-1} + F^\top A^{-1} F)^{-1} \notag \\
    &\ (F^\top A^{-1} S V^{-1} H - F^\top A^{-1}) + A^{-1} + H^\top B^{-1} H \notag \\
    &\ - H^\top V^{-1} S^\top A^{-1} - A^{-1} S V^{-1} H, \label{eq:cov_inv_1} \\
    \Sigma_t^{-1} \boldsymbol{\mu}_t \! &= \! (H^\top \! V^{-1} S^\top \! A^{-1} F - A^{-1} F) (\Sigma_{t-1}^{-1} \! + \! F^\top \! A^{-1} F)^{-1} \notag \\
    &(\Sigma_{t-1}^{-1} \boldsymbol{\mu}_{t-1} \! - \! F^\top \! A^{-1} \! G \bfu_t \! + \! F^\top \! A^{-1} \!  S V^{-1} \! \bfz_t) \! + \! A^{-1} \! G \bfu_t \notag \\
    &- H^\top V^{-1} S^\top A^{-1} G \bfu_t - A^{-1} S V^{-1} \! \bfz_t + H^\top B^{-1} \! \bfz_t. \label{eq:info_mean_1}
\end{align}
To proceed, we need to use the Woodbury identity \cite{petersen2008matrix}:
\begin{align}
    &(A \! + \! CBC^\top)^{-1} \! = \! A^{-1} \! - \! A^{-1} \! C (B^{-1} \! + \! C^\top \! A^{-1} \! C)^{-1} \! C^\top \! A^{-1}, \label{eq:woodbury_1} \\
    &(P^{-1} \! + \! B^\top \! R^{-1} \! B)^{-1} B^\top \! R^{-1} \! = \!  P B^\top (B P B^\top \! + \! R)^{-1}, \label{eq:woodbury_2}
\end{align}
where $P$ and $R$ should be positive definite. 
Focusing on the covariance first, i.e. \eqref{eq:cov_inv_1}, and using \eqref{eq:woodbury_1}, we have
\begin{equation} \label{eq:woodbury_app_1}
    \begin{aligned}
    &A^{-1} - A^{-1} F (\Sigma_{t-1}^{-1} + F^\top A^{-1} F)^{-1} F^\top A^{-1} \\
    &= (A + F \Sigma_{t-1} F^\top)^{-1} \\
    &= (W \! - \! S  V^{-1} \! S^\top \! + \! F  \Sigma_{t-1} \! F^\top)^{-1} \!
    = \! (\Sigma_t^+ \! - \! S V^{-1} S^\top)^{-1},
    \end{aligned}
\end{equation}
where $\Sigma^+_{t}$ is the predicted covariance. To shorten the equations, we define the following identitie:
\begin{equation}
    D = (\Sigma^+_t - S V^{-1} S^\top)^{-1}, \ E = (V - S^\top \Sigma^+_t S)^{-1}.
\end{equation}
%
%
%
%
This way, we get
\begin{align}
    \Sigma_t^{-1} =& D + H^\top V^{-1} S^\top D S V^{-1} H + H^\top V^{-1} H \notag \\
    &-H^\top V^{-1} S^\top D - D S V^{-1} H.
    \label{eq:cov_inv_2}
\end{align}
Using \eqref{eq:woodbury_1} and \eqref{eq:woodbury_2}, we get 
\begin{align}
    &E \! = \! V^{-1} \! + \! V^{-1} \! S^\top \! D S V^{-1}\!\!, \ D \! = \! \Sigma_t^{+ -1} \! + \! \Sigma_t^{+ -1} \! S E S^\top \! \Sigma_t^{+ -1}   \\
    &V^{-1} S^\top D = - E S^\top \Sigma^{+-1}_t, \ D S V^{-1} = - \Sigma_t^{+ -1} S E ,
\end{align}
and \eqref{eq:cov_inv_2} becomes 
\begin{align}
    \Sigma_t^{-1} =& D + H^\top E H + H^\top E S^\top \Sigma_t^{+ -1} + \Sigma_t^{+ -1} S E H \notag \\
    =& \Sigma_t^{+ -1}  +  \Sigma_t^{+ -1}  S E S^\top \Sigma_t^{+ -1} + H^\top E H \notag \\
    &+ H^\top E S^\top \Sigma_t^{+ -1} + \Sigma_t^{+ -1} S E H \notag \\
    =& \Sigma_t^{+ -1} + (H^\top + \Sigma^{+-1}_t S) E (H + S^\top \Sigma^{+ -1}_t).
    \label{eq:cov_inv_3}
\end{align}
Letting $J = H + S^\top \Sigma^{+ -1}_t$ and taking the inverse of the above equation by using \eqref{eq:woodbury_1}, we get
\begin{align}
    \Sigma_t =& \Sigma^+_t - \Sigma^+_t J^\top (J \Sigma^+_t J^\top + E^{-1})^{-1} J \Sigma^+_t \notag \\
    =& \Sigma^+_t - (\Sigma^+_t H^\top + S) (H \Sigma^+_t H^\top + V + HS + S^\top H^\top)^{-1} \notag \\
    & (H\Sigma^+_t + S^\top).
\end{align}
Defining the Kalman gain as
\begin{equation}\label{eq:corr_Kalman_gain}
    K_t = (\Sigma^+_t H^\top + S) (H \Sigma^+_t H^\top + V + HS + S^\top H^\top)^{-1},
\end{equation}
the above equation can be rewritten as 
\begin{equation}
    \Sigma_t = (I - K_t H) \Sigma_t^+ - K_t S^\top,
\end{equation}
%
Then, we show that $\boldsymbol{\mu}_t$ in \eqref{eq:info_mean_1} is equivalent to the following:
\begin{equation}
    \boldsymbol{\mu}_t = \boldsymbol{\mu}_t^+ + K_t (\bfz_t - H \boldsymbol{\mu}^+_t).
    \label{eq:corr_KF_mean_update}
\end{equation}
\eqref{eq:woodbury_2} tells us that
\begin{align}
    &A^{-1} F (\Sigma_{t-1}^{-1} + F^\top A^{-1} F)^{-1} \Sigma_{t-1}^{-1} \boldsymbol{\mu}_t \notag \\
    &= (A + F \Sigma_{t-1} F^\top)^{-1} F \boldsymbol{\mu}_t \notag \\
    &= (W - SV^{-1}S^\top + F \Sigma_{t-1} F^\top)^{-1}   F \boldsymbol{\mu}_t \notag \\
    &= (\Sigma^+_t - S V^{-1} S^\top)^{-1} F \boldsymbol{\mu}_t = D F \boldsymbol{\mu}_t.
    \label{eq:woodbury_app_2}
\end{align}
With \eqref{eq:woodbury_app_1} and \eqref{eq:woodbury_app_2}, we can simplify \eqref{eq:info_mean_1} as
\begin{align}
    \Sigma_t^{-1} \boldsymbol{\mu}_t &= (I - H^\top V^{-1} S^\top) (\Sigma_t^{+} - S V^{-1} S^\top)^{-1} \notag \\
    &\quad\; (F \boldsymbol{\mu}_{t-1} + G\bfu_t - S V^{-1} \bfz_t) + H^\top V^{-1} \bfz_t \notag \\
    &= (I \! - \! H^\top \! V^{-1} \! S^\top) \! D (\boldsymbol{\mu}^+_{t} \! - \! S V^{-1} \! \bfz_t) \! + \! H^\top \! V^{-1} \bfz_t. \label{eq:info_mean_2}
\end{align}
Then, our goal is to show that $\boldsymbol{\mu}_t$ in \eqref{eq:corr_KF_mean_update} and \eqref{eq:info_mean_2} are the same, i.e. the above equation can be written as:
\begin{equation}\label{eq:mean_update_relation}
    \Sigma_t^{-1} \boldsymbol{\mu}_t = \Sigma_t^{-1}  \boldsymbol{\mu}_t^{+} + \Sigma_t^{-1} K_t (\bfz_t - H \boldsymbol{\mu}^+_t) 
\end{equation}
The Kalman gain can be rewritten by using \eqref{eq:woodbury_1} and \eqref{eq:woodbury_2} as
\begin{align}
    K_t &= (\Sigma^+_t H^\top + S) (H \Sigma^+_t H^\top + V + HS + S^\top H^\top)^{-1} \notag \\
    &= \Sigma^+_t J^\top (J \Sigma_t^+ J^\top + E^{-1})^{-1} \notag \\
    &= (\Sigma_t^{+ -1} + J^\top E J)^{-1} J^\top E \notag \\
    &= \Sigma_t J^\top E.
\end{align}
Therefore,
\begin{align}
    &\Sigma^{-1}_t K_t (\bfz_t - H\boldsymbol{\mu}_t) \notag \\
    =& J^\top E \bfz_t - J^\top E H \boldsymbol{\mu}_t \notag \\
    =& H^\top \! E  \bfz_t \! + \! \Sigma_t^{+ -1} \! S E \bfz_t \! - \! H^\top \! E  H \! \boldsymbol{\mu}_t \! - \! \Sigma_t^{+ -1} \! S E H \boldsymbol{\mu}_t.
\end{align}
Substituting the first line of \eqref{eq:cov_inv_3} into the right hand side of \eqref{eq:mean_update_relation}, we get
\begin{align}
    &\Sigma_t^{-1}  \boldsymbol{\mu}_t^{+} + \Sigma_t^{-1} K_t (\bfz_t - H \boldsymbol{\mu}^+_t) \notag \\
    =& D \boldsymbol{\mu}_t + H^\top E H \boldsymbol{\mu}_t + H^\top E S^\top \Sigma_t^{+ -1} \boldsymbol{\mu}_t + \Sigma_t^{+ -1} S E H \boldsymbol{\mu}_t \notag \\
    &+ H^\top \! E  \bfz_t \! + \! \Sigma_t^{+ -1} \! S E \bfz_t \! - \! H^\top \! E  H \! \boldsymbol{\mu}_t \! - \! \Sigma_t^{+ -1} \! S E H \boldsymbol{\mu}_t \notag \\
    =& D \boldsymbol{\mu}_t + H^\top E S^\top \Sigma_t^{+ -1} \boldsymbol{\mu}_t + H^\top \! E  \bfz_t  +  \Sigma_t^{+ -1}  S E \bfz_t \notag \\
    =& D \boldsymbol{\mu}_t - H^\top V^{-1} S^\top D \boldsymbol{\mu}_t \notag \\
    &+ H^\top V^{-1} \bfz_t + H^\top V^{-1} S^\top D S V^{-1} \bfz_t -DSV^{-1} \bfz_t \notag \\
    =& (I \! - \! H^\top \! V^{-1} \! S^\top)  D (\boldsymbol{\mu}^+_{t} \! - \! S V^{-1} \! \bfz_t) + H^\top \! V^{-1} \bfz_t,
\end{align}
which aligns with \eqref{eq:info_mean_1}, thus completing the proof.

In summary, given estimation from the previous timestep $\boldsymbol{\mu}_{t-1}, \Sigma_{t-1}$, control input $\bfu_t$, and observation $\bfz_t$ at time $t$, the Kalman filter with correlated motion and observation noise is shown below.

\noindent Prediction:
\begin{align}
    \bfmu_{t}^+ &= F \bfmu_{t-1} + G\bfu_t,\\
    \Sigma_{t}^+ &= F \Sigma_{t-1} F^\top + W.
\end{align}

\noindent Update:
\begin{align}
    K_t &= (\Sigma^+_t H^\top \! + \! S) (H \Sigma^+_t H^\top \! + \! V \! + \! HS \! + \! S^\top H^\top)^{-1}, \\
    \bfmu_{t} &= \bfmu_{t}^+ + K_{t} (\bfz_{t} - H \bfmu_{t}^+),\\
    \Sigma_{t} &= (I - \bar{K}_{t} H) \Sigma_{t}^+ - K_t S^\top,
\end{align}
which aligns with \cite[Table~5.1]{crassidis2004optimal}.
\end{document}